\documentclass{article}
\usepackage[utf8]{inputenc}
\usepackage[english]{babel}
\usepackage{graphicx}
\usepackage[dvipsnames,table,usenames]{xcolor}
\usepackage{booktabs}

\usepackage[round]{natbib}
\usepackage{physics}

\usepackage{soul}
\usepackage{enumerate}
\usepackage{subcaption}
\usepackage{amsmath, amsfonts, amssymb, amstext, amsthm, mathtools,gensymb}
\usepackage{thmtools,thm-restate}
\usepackage{algorithmic,}
\usepackage{algorithm}

\usepackage{tikz}
\usepackage{pgfplots}
\pgfplotsset{compat=1.5}
\usetikzlibrary{arrows}
\usetikzlibrary{matrix}

\usepackage{textcomp}
\usepackage{gensymb} 
\usepackage{diagbox}
\usepackage[colorlinks=true,linkcolor=black,anchorcolor=black,citecolor=black,filecolor=black,menucolor=black,runcolor=black,urlcolor=black,pdfencoding=auto]{hyperref}
\usepackage[noabbrev]{cleveref}
\usepackage{todonotes}
\usepackage{lipsum}  
\usepackage{multirow} 
\usepackage{adjustbox}

\usepackage{comment}
\usepackage{nomencl}
\usepackage{ifthen}
\usepackage[hybrid]{markdown}
\usepackage[leftmargin=-1cm]{mdframed}
\usepackage{xifthen} 

\newtheorem{theorem}{Theorem}

\newtheorem{definition}[theorem]{Definition}

\def\Dcal{\mathcal{D}}

\def\Fbb{\mathbb{F}}
\def\Hbb{\mathbb{H}}

\def\Pcal{\mathcal{P}}

\def\Qcal{\mathcal{Q}}

\def\Ucal{\mathcal{U}}
\def\Vcal{\mathcal{V}}
\def\Xcal{\mathbb{X}}
\def\Ycal{\mathbb{Y}}

\def\Dcalx{\mathcal{D}_{\Xcal}}
\def\Dcaly{\mathcal{D}_{\Ycal}}
\def\Dcalyx{\mathcal{D}_{\Ycal|\x}}
\def\Dcalxy{\mathcal{D}_{\Xcal|y}}
\def\Scal{{\mathcal{S}}}
\def\Scalx{{\Scal_{\Xcal}}}
\def\Scaly{{\Scal_{\Ycal}}}

\def\Scalxy{{\Scal_{\Xcal|y}}}
\def\Tcal{{\mathcal{T}}}

\def\Tcalxy{{\Tcal_{\Xcal|y}}}
\def\Tcalx{{\Tcal_{\Xcal}}}
\def\Tcaly{{\Tcal_{\Ycal}}}

\def\Rset{{\mathbb R}}
\def\Nset{{\mathbb N}}

\def\ones{\mathbf 1}

\def\I{{\bf I}}
\def\P{{\bf P}}

\def\P{{\bf P}}

\def\Ybf{{\bf Y}}

\def\p{{\bf p}}

\def\v{{\bf v}}
\def\u{{\bf u}}

\def\x{{\bf x}}
\def\y{{\bf y}}

\def\rbf{{\bf r}}

\newcommand{\mubf}{\boldsymbol{\mu}}

\def\alphabf{{\boldsymbol{\alpha}}}
\def\betabf{{\boldsymbol{\beta}}}
\def\epsbf{{\boldsymbol{\epsilon}}}
\def\etabf{{\boldsymbol{\eta}}}

\def \ie{\emph{i.e. }} 
\def \eg{\emph{e.g. }}

\newcommand{\esp}[2]{\mathbb{E}_{#1}\left[#2\right]}
\newcommand{\proba}[2]{\mathbb{P}_{#1}\left[#2\right]}

\def\argmax{\operatorname{argmax}}
\def\argmin{\operatorname{argmin}}
\newcommand{\sgn}[1]{\operatorname{sgn}\left(#1\right)}

\newcommand{\supp}[1]{\operatorname{supp} #1}

\newcommand{\risk}[3][]{  \ifthenelse{\isempty{#1}}
                        {\mathfrak{L}_{#2}\left(#3\right)}
                        {\mathfrak{L}^{#1}_{#2}\left(#3\right)}}
\newcommand{\urisk}[3][]{\ifthenelse{\isempty{#1}}
                            {\mathfrak{C}_{#2}\left(#3\right)}
                            {\mathfrak{C}^{#1}_{#2}\left(#3\right)}
                            }

\def\l01{l_{01}}

\def\HdH{{\Hbb\Delta\Hbb}}

\def \IPM{\operatorname{IPM}}
\def \IMD{\operatorname{IMD}}

\usepackage{arxiv}

\title{Connecting sufficient conditions for domain adaptation: source-guided uncertainty, relaxed divergences and discrepancy localization}

\author{Sofien Dhouib\\
	Department of Computer Science\\
	University of Tübingen\\
	Germany \\
	\texttt{sofiane.dhouib@inf.uni-tuebingen.de} \\
	\And Setareh Maghsudi\\
	Department of Computer Science\\
	University of Tübingen\\
	Germany\\
	\texttt{setareh.maghsudi@uni-tuebingen.de}}

\date{}

\renewcommand{\shorttitle}{DA connections}

\begin{document}
\maketitle
\begin{abstract}
Recent advances in domain adaptation establish that requiring a low risk on the source domain and equal feature marginals degrade the adaptation's performance. At the same time, empirical evidence shows that incorporating an unsupervised target domain term that pushes decision boundaries away from the high-density regions, along with relaxed alignment, improves adaptation. In this paper, we theoretically justify such observations via a new bound on the target risk, and we connect two notions of relaxation for divergence, namely $\beta-$relaxed divergences and localization. This connection allows us to incorporate the source domain's categorical structure into the relaxation of the considered divergence, provably resulting in a better handling of the label shift case in particular.
\end{abstract}
\section{Introduction}
Supervised learning algorithms are prone to failure when the training and testing distributions are different. That arises in several real world applications such as speech recognition and computer vision, due to changes in the data collection process for example. As solving this problem by collecting more data might be problematic due to the potential cost of the labeling process, the \emph{Domain Adaptation} (DA) field \citep{pan_survey_2010,weiss_survey_2016} has emerged to tackle the issue in attempt to transfer the knowledge acquired on the labeled training set, stemming from a source distribution, to a partially or totally unlabeled testing set, corresponding to a target distribution.
\par Over the last decade, DA has been the focus of several lines of work. On the theoretical level, in the context of tackling the distribution shift problem \citep{quinonero2008dataset}, one salient idea is to bound the risk on the target distribution by quantities reflecting the performance on the source domain along with its relatedness to the target \citep{ben-david_analysis_2007,mansour_domain_2009,cortes_domain_2011,ben-david_domain_2014,cortes_domain_2014,germain_pac-bayesian_2013,germain_new_2016,zhang_bridging_2019}. We refer the interested reader to \citet{redko_survey_2020} for a more exhaustive account in this regard. On the algorithmic level, early approaches aim at aligning distributions on the feature level \citep{blitzer_biographies_2007,daume_iii_frustratingly_2009,fernando_subspace_2014} or at the instance level via reweighting \citep{shimodaira_covariate_2000,sugiyama_covariate_2007,huang2007correcting,cortes_learning_2010}. The goal behind such an alignment is to reduce some dissimilarity measure between the domains, such as the Wasserstein distance \citep{courty_optimal_2016,courty_joint_2017}, the Maximum Mean Discrepancy \citep{huang2007correcting,gong_connecting_2013}, or the distance between covariance matrices \citep{sun_return_2016}, to name a few. The recent reviews of \citet{kouw_review_2019,zhang_transfer_2019,zhuang2020comprehensive} provide an excellent overview of the different methods. More recently, the emergence of deep learning \citep{goodfellow2016deep} resulted in a family of methods looking for a feature representation that not only is discriminative for classes on the source domain, but that is also domain-agnostic \citep{ganin_domain-adversarial_2016,long2018conditional,shu_dirt-t_2018}. Such approaches have proven their effectiveness especially for computer vision \citep{csurka2017domain}. We refer the interested reader to \citet{wang_deep_2018,wilson2020survey} for reviews on deep DA.
\par Nevertheless, domain alignment imposes the cross-labeling risk, resulting in a correspondence between instances having different labels \citep{mehra2021understanding}. That happens, for example, when the label marginals vary between the two domains. Provably, this situation deteriorates the adaptation performance when combined with a good performance on the source domain \citep{zhao_learning_2019,wu2019domain,le2021lamda}. As a result, two directions to tackle this problem have been recently studied. On the one hand, some lines of work relax the requirement of equality of distribution when trying to align them \citep{johansson_support_2019,wu2019domain,zhang2020localized,tong2022adversarial}, whereas other approaches jointly align the label and feature marginals in an attempt to circumvent the label shift problem \citep{redko2019optimal,tachet2020domain} and its generalizations \citep{rakotomamonjy2021optimal,kirchmeyer2022mapping} allowing an additional shift in label-conditionals. These approaches, however, are disconnected and more understanding on their relations is needed. Moreover, they only handle the problem of strictness of the divergence, whereas the recent negative results we mentioned also point out to the role of requiring a good performance on the source domain. In this regard, recent deep learning approaches are increasingly using an unsupervised loss on the target domain to promote a more class-discriminative structure in addition to requiring a good performance on the source  \citep{shu_dirt-t_2018,saito_semi-supervised_2019,tan2020class,kirchmeyer2022mapping,tong2022adversarial}. To the best of our knowledge, only a few papers, including \citet{germain_new_2016} and \citet{morerio_correlation_2017}, deliver theoretical evidence of the benefits of such terms.

Against this background, in this paper we provide the following contributions:
\begin{itemize}
    \item We theoretically prove the utility of the minimizing the uncertainty of the considered scoring function in its predictions on the target domain while being guided by the predictions on the source domain. Our result tightens a broad class of previously established DA bounds in the sense that it only requires these bounds to involve a risk on the source domain.
    \item We introduce a new discrepancy between measures that generalizes Integral Probability Metrics \citep{zolotarev1984} when the compared measures do not have the same mass, and use it to connect two notions of relaxed dissimilarity between domains, namely \emph{localized discrepancies} \citep{zhang2020localized} and \emph{$\beta-$admissible distances} \citep{wu2019domain}. We further harness our established link in order to incorporate the source domain's observable categorical structure into the relaxation, leading to an additional connection to class re-weighting methods.
    \item Depending on the choice of the functional space defining our discrepancy measure, we revisit previously established results. In particular we theoretically justify approaches that rely on the extremely relaxed requirement of confusion of supports rather than of distributions.
    \item We illustrate the benefits of taking the categorical structure of the source into account when relaxing the discrepancy between distributions, via experiments on toy datasets in the particular case of Wasserstein distances.
\end{itemize}

\paragraph{Outline of the paper} After introducing the problem setup and the notations in \Cref{sec:problem_setup}, we prove in \Cref{sec:source_guided_uncertainty} the theoretical interest of enforcing a scoring function at hand to be confident in its predictions. Then, we specialize our study to bounds involving a divergence term and a joint risk in \Cref{sec:weak_alignment_localization}, where we link previously introduced notions of relaxation. and we extend them by incorporating the source domain's categorical structure. \Cref{sec:special_cases} is dedicated to revisiting previously established DA results in the light of our theoretical results. Finally, we illustrate the interest of incorporating the categorical structure over classic relaxation in \Cref{sec:experiments}.
\section{Problem setup and notations}
\label{sec:problem_setup}
We consider a multi-class domain adaptation setting, where the feature space is $\Xcal$, a compact subset of $\Rset^p$ ($p \in \Nset^*$) and the label space is $\Ycal = \{y_1, \cdots, y_K\}$, where the different classes are encoded as the basis vectors of $\Rset^K$ (one hot encoding), unless specified otherwise. The source and target domains correspond to two joint distributions $\Scal$ and $\Tcal$ over $\Xcal \times \Ycal$. For any probability distribution $\Dcal$ over $\Xcal\times\Ycal$, we denote by $\Dcalx,\Dcaly,\Dcalxy,\Dcalyx$ its feature- and label- marginals, and label- and feature- conditionals respectively, for $\x\in\Xcal, y\in\Ycal$\footnote{Vectors are denoted in bold lower case font.}. In particular, for the class conditional distributions, by abuse of notation we will denote $\Dcal_{\Xcal|k}$ instead of $\Dcal_{\Xcal|y_k}$. When we refer to labeling functions $f_\Scal$ (resp. $f_\Tcal$) of the source (resp. target) domain, we mean a function that outputs a vector over the $K-$dimensional probability simplex, which can be degenerate to indicate only one class in the case of deterministic labeling. All of the measures we consider over $\Xcal$ have densities, \ie, they are absolutely continuous with respect to the Lebesgues measure.
\par Concerning the performance of classification, we consider classifiers as functions in $\Ycal^\Xcal$ that are typically selected from a hypothesis space $\Hbb\subsetneq\Ycal^\Xcal$, and scoring functions as functions in $(\Rset^K)^\Xcal$. A loss function is any function $l: \Rset^K\times \Rset^K \to \Rset_+$, typically taking a scoring function or a classifier for the first argument, and a classifier for the second, and verifying $l(y,y) = 0\ \forall y \in \Ycal$. Given a scoring function $g$, we denote its associated classifier as $h_g$, \ie $h_g(\x) = y_k$ with $k \in \argmax g(\x)$, and assume that for all the loss functions we consider, we have $h_g(\x) \in \argmin_{y \in \Ycal} l(g(\x),y)$ for any $\x$ in $\Xcal$. For the sake of conciseness, we consider that any classifier is also a scoring function, with $h_g = g$. Finally, we define the $l-$risk of a scoring function $g$ over a distribution $\Dcal$ over $\Xcal \times \Ycal$ as $\risk{\Dcal}{g} \coloneqq \esp{(\x,y)\sim\Dcal}{l(g(\x),y)}$, and we extend it to the disagreement of $g$ with a classifier $h$ as $\risk{\Dcal}{g,h} \coloneqq \esp{\x\sim\Dcalx}{l(g(\x),h(\x))}$.

\section{Role of the confidence of a scoring functions in its predictions on the target domain}
\label{sec:source_guided_uncertainty}
Given a scoring function $g: \Xcal \to \Rset^K$, one approach to measure its uncertainty in its predictions $h_g$, using a loss function $l$, is to compute the following quantity.
\begin{definition}[Uncertainty of a scoring function]\label{def:uncertainty}
    Given a scoring function $g$ with $h_g \in \Hbb$, we define its uncertainty in its predictions over a distribution $\Dcal$ over $\Xcal \times \Ycal$ by
    \begin{equation}
    \inf_{h \in \Hbb} \risk{\Dcal}{g,h} = \risk{\Dcal}{g,h_g}.
\end{equation}
\end{definition}
Below, we give two known examples of this quantity for different choices of the loss function $l$.
\begin{restatable}[Cross entropy loss]{example}{explCrossEntropyLoss}
\label{expl:crossEntropyLoss}
In this case, $g(\x)$ is in the $K-$dimensional probability simplex (after applying the softmax function). We then have $l(g(\x), h_g(\x)) = H_\infty(g(\x))$,\footnote{The proofs of all theoretical claims can be found in the appendix.}
    where $H_\alpha$ denotes the Renyi entropy \citep{renyi1961measures} defined as $H_\alpha(g(\x)) = \frac{1}{1-\alpha}\log \left(\sum_i (g(\x))_i^\alpha\right)$, which is equal to the Shannon entropy when $\alpha = 1$. Since we have $H_\infty(g(\x)) \leq H_\alpha(\x)\quad \forall \alpha \geq 0$, having a small conditional entropy of $g(\x)$ in particular results in a small uncertainty of $g(\x)$. Minimizing the conditional entropy, borrowed from the semi-supervised learning literature \citep{grandvalet2005semi,erkan2010semi}, is now a standard approach in domain adaptation \citep{shu_dirt-t_2018,liang2021source,kirchmeyer2022mapping}.
\end{restatable}
\begin{restatable}[Hinge loss]{example}{explHingeLoss}\label{expl:hingeLoss}
The binary SVM \citep{boser_training_1992,cortes_support-vector_1995} problem can be written as minimizing the regularized risk of the hinge loss function $l(g(\x),y) = (1 - yg(\x))_+$, where $g$ is a linear classifier and $y \in \{-1,1\}$. Computing the uncertainty of $g$ at $\x$ yields  $(1-\abs{x})_+$, the latter being an unsupervised loss appearing in the transductive SVM formulation \citep{chapelle2005semi,collobert2006large}. A smooth analogue of this quantity also appears in \citet{germain_new_2016}.
\end{restatable}
    

\par In the absence of labeled target data, the best one can hope for by obtaining a confident classifier is a clustering in which the scoring function identifies the classes up to a permutation. In what follows, we combine the previous quantity with supervision from the source domain to obtain a source-guided measure of uncertainty.
\begin{definition}[Source-guided uncertainty]
    Let $\Hbb$ be a hypothesis space, and let $l^1$ and $l^2$ be two loss functions with associated risks $\risk[1]{\Dcal}{.}$ and $\risk[2]{\Dcal}{.}$ for a distribution $\Dcal$. The source-conditioned confidence of a function $g \in \Ycal^\Xcal$ associated to the two previous losses is:
    \begin{equation}
        \urisk[1,2]{\Hbb}{g} = \inf_{h\in \Hbb}\risk[1]{\Tcal}{g,h} + \risk[2]{\Scal}{h}.
    \end{equation}
    In the case where $l^1 =l^2 = l$, we simply denote $\urisk{\Hbb}{h}$, where $l$ is clear from the context.
\end{definition}
Compared with \Cref{def:uncertainty}, we have an additional term that can be thought of as a regularization forcing $h$ to be compatible with the labels on the source domain. A second interpretation of the source-guided uncertainty is that when $l^1 = l^2 = l$, it is equal to the ideal joint $l-$risk corresponding to the class of functions $\Hbb$, for a target domain labeled by $g$. In particular, for $g = f_\Tcal$, it is equal to the ideal joint risk \citep{ben-david_theory_2010,acuna2021f,zhong2021does}. 
\par The next proposition formalizes some properties of the source-guided uncertainty and helps to understand when it is small.
\begin{restatable}[Properties of the source-guided uncertainty]{proposition}{propSourceGuidedUncertainty}\label{prop:sourceGuidedUncertainty}
    The source-guided uncertainty of a scoring function verifies the following properties:
    \begin{enumerate}
        \item If $g$ is a scoring function and $h_g \in \Hbb$, then
            $\urisk[1,2]{\Hbb}{g} \leq \risk[1]{\Tcal}{g,h_g} + \risk[2]{\Scal}{h_g}$. In particular, if $g \in \Hbb$, then $\urisk[1,2]{\Hbb}{g} \leq \risk[2]{\Scal}{g}$.
        \item If $\Hbb \subseteq \tilde\Hbb \subseteq \Ycal^\Xcal$, and if $l^1(a,a) = 0$, then
            $\inf_{g \in \tilde \Hbb}\urisk[1,2]{\Hbb}{g} = \inf_{h \in \Hbb}\risk[2]{\Scal}{h}$.
        \item If $l = l^1 = l^2$ obeys the triangle inequality, then $\urisk{\Hbb}{h} = \risk{\Scal}{h}$ for $h \in \{h_\Scal, h^*\}$, where $h_\Scal \in \argmin \risk{\Scal}{h}$ and $h^*\in\argmin_{h \in \Hbb}\risk{\Tcal}{h} + \risk{\Scal}{h}$.
    \end{enumerate}
\end{restatable}
In \Cref{prop:sourceGuidedUncertainty}, the first point justifies the intuition that the source-guided uncertainty is small whenever the scoring function $g$ performs well on $\Scal$ (in terms of the risk of $h_g$), while having a low uncertainty on $\Tcal$ (\ie, a low $\risk[2]{\Tcal}{g, h_g}$). And when specialized to the case where $g$ is a classifier in $\Hbb$, it provides a lower bound for the risk on the source domain. The second point shows that when the space of scoring functions is rich enough to contain classifiers, the source-guided uncertainty coincides with the best achievable $l^2-$risk on the source domain. The last point states an interesting fact about the case of equality with the source risk, achieved for the ideal joint hypothesis that normally requires access to target labels, and for the best source hypothesis.
\par Our next result shows the role of the source-guided uncertainty in tightening \emph{any} domain adaptation bound that comprises the risk of the considered classifier on the source domain.
\begin{restatable}[Tightening bounds on the target risk]{proposition}{propSourceGuidedUncertaintyBound}\label{prop:sourceGuidedUncertaintyBound}
Let $g$ be a scoring function, and let $l^1$ be a loss function, verifying $l^1(u,y_1) - l^2(y_2,y_1) \leq l^1(u,y_2)\ \forall u \in \Rset^K, y_1, y_2 \in \Ycal$. Then given a bound on the target risk of the following form:
\begin{equation}
    \risk[2]{\Tcal}{h} \leq \risk[2]{\Scal}{h} + A(\Tcal,\Scal); \quad \forall h \in \Hbb,
\end{equation}
where $A(\Tcal,\Scal)$ reflects a relatedness between the domains' joint distributions, we have for any scoring function $g$
\begin{equation}
    \risk[1]{\Tcal}{g} \leq \urisk[1,2]{\Hbb}{g} + A(\Tcal,\Scal).
\end{equation}
\end{restatable}

The statement of \Cref{prop:sourceGuidedUncertaintyBound} holds for a large class of domain adaptation bounds on the target risk, and according to point 1 of \Cref{prop:sourceGuidedUncertainty}, when $g \in Hbb$, it theoretically shows that the source-guided uncertainty provides a bound that is tighter than classic bounds involving the risk on the source domain. For scoring functions, this tightness may justify the empirical effectiveness of the conditional entropy minimization (Example \ref{expl:crossEntropyLoss}), as done in  \citep{shu_dirt-t_2018,kirchmeyer2022mapping}.
Beyond the case of $g \in \Hbb$, we have a bound that holds even when space $\Hbb$ is richer than the hypothesis space from which we select $g$. In the next section, we will use this fact to derive bounds with several families of divergences between the two domains, after we specialize $A(\Tcal,\Scal)$ to the form of the sum of a divergence term and a joint minimum risk \citep{ben-david_theory_2010,acuna2021f}.  Concerning the assumption on the loss function, it reduces to the triangle inequality when both are equal. It also holds, for example, for $l^1$ and $l_2$ chosen as the cross-entropy loss and the $L^1$ loss up to a multiplicative factor.

\section{Relating weak alignment and localization}
\label{sec:weak_alignment_localization}
So far, we theoretically justified the role of the source-guided uncertainty term on the target domain. Now we focus on a special class of bounds that assume low value of the ideal joint risk. Our results will involve relaxed divergences that do not require equality of the two distributions to be null, a property shared by the $\HdH-$distance \citep{ben-david_theory_2010} and its generalization the $l-$discrepancy \citep{mansour_domain_2009}. They will also be asymmetric, as in the case of the recently introduced $\beta-$admissible distances \citep{wu2019domain} and localized discrepancies \citep{zhang2020localized}. In order to proceed, we begin by the following definitions for localized sets.
\begin{definition}[$(\epsilon,\Scalx)-$localized space of nonnegative functions]\label{def:localizedSpace}
    Let $\Fbb$ be a set of nonnegative functions. For $\epsilon \geq 0$, the $(\epsilon,\Scalx)$ localized subset of $\Fbb$ is defined as 
    \begin{equation}
        \Fbb_\epsilon = \{f \in \Fbb; \quad \esp{\Scalx}{f} \leq \epsilon\}
    \end{equation}
\end{definition}
An instance of \Cref{def:localizedSpace} allows to define the \emph{localized discrepancy} introduced in \citet{zhang2020localized}, by choosing
$\Fbb = l(\Hbb,f_\Scal) \coloneqq \{l(h,f_\Scal); h \in \Hbb\}$, where $l$ is a loss function. 
As indicated in \citet{zhang2020localized}, the previous example is motivated by the following observation: If the ideal joint risk is $\lambda = \inf_{h \in \Hbb}\risk{\Tcal}{h} + \risk{\Scal}{h}$, then its value will not change when restricting the choice of $h \in \Hbb$ to hypotheses achieving risk at most $\lambda$ on the source domain. In other words, among all of the hypotheses that perform well on $\Scal$, there is one that has a low risk on both domains.
\par In addition the previous notion of localization, we introduce a family of dissimilarities between measures over the feature space $\Xcal$ in the following definition. 
\begin{definition}[Integral Measure Discrepancy]
\label{def:IMD}
Let $\Fbb$ be a family of nonnegative functions over $\Xcal$, containing the null function. The Integral Measure Discrepancy (IMD) associated to $\Fbb$ between two nonnegative finite measures $\Qcal_1$ and $\Qcal_2$ over $\Xcal$ is
\begin{equation*}
	\label{eq:IMD}
	\IMD_{\Fbb}(\Qcal_1, \Qcal_2) \coloneqq \sup_{f \in \Fbb}\int f \dd\Qcal_1 - \int f \dd\Qcal_2.
\end{equation*}
\end{definition}
The IMD is obviously a generalization of Integral Probability Metrics (IPM) \citep{zolotarev1984,mueller1997integral} to measures with possibly different masses. The interest in distances between measures with different masses is not new and has been the topic of \citet{steerneman1983total} for the total variation distance and Hellinger distances, \citet{benamou2015iterative} to define Bregman projections and   \citet{chizat2018unbalanced,chizat2018scaling,liero_optimal_2018, fatras2021unbalanced} for the unbalanced optimal transport problem, to name a few. Some of the IMD's basic properties are given by the following proposition.
\begin{restatable}[Properties of the IMD]{proposition}{IMDProperties}
\label{prop:IMDProperties}
	The IMD is nonnegative, satisfies the triangle inequality, and we have $\IMD_{\Fbb}(\Qcal_1,\Qcal_2) = 0$ if $\Qcal_1 \leq \Qcal_2$ (\ie, $q_1 \leq q_2$ when $q_1$ and $q_2$ are the densities of $\Qcal_1$ and $\Qcal_2$). Moreover, for $\Fbb$ rich enough (\ie, containing the continuous functions or the indicator functions), we have $\IMD(\Qcal_1,\Qcal_2)= 0$ only if $\Qcal_2 \geq \Qcal_1$.
\end{restatable}
In particular, \Cref{prop:IMDProperties} shows that the IMD is asymmetric. Indeed, it is sufficient to take a measure $\Qcal$ that is non identically null, then $\IMD_\Fbb(\Qcal,2\Qcal) = 0$ while $\IMD(2\Qcal,\Qcal)>0$ whenever $\Fbb$ contains a function with $\int f \dd\Qcal>0$.
\par With the previously introduced quantities, we are now ready to state our domain adaptation bound involving the IMD and localization.
\begin{restatable}
	{proposition}{propBoundLocalizedIMD}
\label{prop:BoundLocalizedIMD}
Let $\Hbb$ be a hypothesis space, $g$ a scoring function not necessarily in $\Hbb$, and $l$ a loss function verifying the triangle inequality. Assume that $l(\Hbb,\Hbb)\coloneqq \{l(h_1(\cdot),h_2(\cdot)); h_1, h_2 \in \Hbb\} \subseteq \Fbb,$ where $\Fbb$ is a set of bounded nonnegative functions. Also, for any $r\geq0$, we consider the localized hypothesis space $\Hbb^r \coloneqq \{h \in \Hbb; \risk{\Scal}{h} \leq r\}$.
\par Then, for any $r_1, r_2 \geq 0$,
        \begin{equation}
        \risk{\Tcal}{g} \leq \urisk{\Hbb^{r_1}}{g} + \IMD_{\Fbb_{r_1 + r_2}}(\Tcalx,\Scalx) + \inf_{h \in \Hbb^{r_2}}\risk{\Tcal}{h} + \risk{\Scal}{h}.
        \end{equation}
\end{restatable}
\Cref{prop:BoundLocalizedIMD} is a generalization of the ones in \citet{ben-david_theory_2010,zhang_bridging_2019,acuna2021f}, as it involves an ideal joint risk. It is based on  localized hypothesis spaces that have been considered for DA in \citet{zhang2020localized}; however, our result extends the latter in the following aspects. First, as we pointed out in \Cref{sec:source_guided_uncertainty}, the concerned hypothesis $g$ does not have to be in $\Hbb$, and our bound is tighter than using a source loss whenever $g \in \Hbb^r$, a condition that is required in the analogous result of \citet{zhang2020localized}. Second, the space $\Fbb$ we consider to define the discrepancy term is more general as it includes the result of \citet{zhang2020localized} for $\Fbb = \HdH$.
\par In the next proposition, we leverage the Lagrange duality to prove an upper bound for the localized IMD, that will hold tightly under mild conditions. It will be the key to establish the connection between localization \citep{zhang2020localized} and $\beta-$admissible distances \citep{wu2019domain}.

\begin{restatable}[Duality for localized IMD]{proposition}{propIMDDuality}
\label{prop:IMDDuality}
    Assume $\Fbb$ is a space of nonnegative functions over $\Xcal$. Then
    \begin{equation}
        \forall \epsilon\geq 0,\quad \IMD_{\Fbb_\epsilon}(\Tcalx,\Scalx) \leq \inf_{\alpha \geq 0}\IMD_{\Fbb}(\Tcalx, (1+\alpha)\Scalx) + \epsilon\alpha.
    \end{equation}
    Moreover, if $\Fbb$ is convex, $\IMD_{\Fbb}(\Tcalx,\Scalx)$ is finite and $\epsilon>0$, then we have an equality, and the infimum at the right hand side is achieved for some $\alpha^*\geq 0$.
\end{restatable}
In spite of having a lower value than IPM's (since $\Fbb_\epsilon\subseteq \Fbb$), a localized IMD is null if and only if both distributions are identical whenever $\Fbb$ is rich enough as expressed in the following corollary.
\begin{restatable}{corollary}{corFLocalizedIsDivergence}\label{cor:FLocalizedIsDivergence}
	If $\Fbb$ is rich enough, then for any $\epsilon>0$, $\IMD_{\Fbb_\epsilon}(\Tcalx,\Scalx) =0$ if and only if $\Tcalx = \Scalx$.
\end{restatable}
 To guard against the drawbacks of strict alignment implying that $\Tcalx = \Scalx$ \citep{zhao_learning_2019,wu2019domain,le2021lamda}, we apply \Cref{prop:IMDDuality} via bounding the infimum over the choice of $\alpha\geq 0$ by an arbitrary $\beta\geq 0$, thus leading to the following corollary.
\begin{restatable}{corollary}{corBoundLocalizedIMDBeta}
\label{cor:BoundLocalizedIMDBeta}
    With the assumptions of \Cref{prop:BoundLocalizedIMD}, for a scoring function $g$, any $r_1, r_2, \beta \geq 0$, we have
        \begin{equation}
        \risk{\Tcal}{g} \leq \urisk{\Hbb^{r_1}}{g} + \IMD_\Fbb(\Tcalx, (1+\beta)\Scalx) + \beta (r_1 + r_2)+ \inf_{h \in \Hbb^{r_2}}\risk{\Tcal}{h} + \risk{\Scal}{h}
        \end{equation}
\end{restatable}
 Corollary \ref{cor:BoundLocalizedIMDBeta} is reminiscent of the DA bound in \citet{wu2019domain}, and shows the interest of $\beta-$admissible distances through a different derivation. Indeed, we do not have a $(1+\beta)$ factor multiplying the source risk of $g$ as in their work. Moreover, as implied by \Cref{prop:IMDProperties}, the discrepancy $\IMD_\Fbb(\Tcalx,(1+\beta)\Scalx)$ is always nonnegative and is null if and only if $t(\x) \leq (1+\beta)s(\x)$ almost surely for $\Fbb$ rich enough. The latter condition is the bounded ratio condition required to define $\beta-$admissible distances \citep{wu2019domain}. However, when $\Fbb$ is more restricted, it generalizes it in the same way that Integral Probability Metrics do not need the density ratio to be defined, compared to Csiszár divergences \citep{csiszar1967information}.
 \par In the next section, we will utilize the link between localization and $\beta-$admissible in order to extend these notions by taking into account the categorical structure of the source domain.
\section{Incorporating the source's categorical structure}\label{sec:categorical_structure}
A limitation of the localization introduced \citet{zhang2020localized} is that it is a global consideration of the performance on the source domain, without a finer look at the performance per class. In the following definition, we introduce a stricter notion of localization that requires the bound on the expectation to hold per class.
\begin{definition}[$(\epsbf,\Scalxy)-$localized space of nonnegative functions]
\label{def:localizedSpacePerClass}
    Let $\Fbb$ be a set of nonnegative measurable functions over $\Xcal$. For $\epsbf = (\epsilon_1, \cdots, \epsilon_K) \geq 0$, the $(\epsbf,\Scalxy)$ localized subset of $\Fbb$ is defined as
    \begin{equation}
        \begin{aligned}
            \Fbb_\epsbf &= \{f\in\Fbb;\quad\esp{\Scal_{\Xcal|k}}{f} \leq \epsilon_k\quad\forall 1\leq k\leq K\}
        \end{aligned}
    \end{equation}
\end{definition}
Choosing $\Fbb = l(\Hbb,f_\Scal)$ corresponds to a stricter notion of localization than the one introduced in \citet{zhang2020localized}, as it requires the hypothesis to have a good performance for every class on the source domain. The intuition behind this requirement is that the considered hypothesis space must contain an ideal joint hypothesis with a performance that should not depend on the class proportions  of the source domain. Requiring a low classification risk per class has already been theoretically considered in the definition of the \emph{Balanced Error Rate} of \citet{tachet2020domain}, in which the authors use it to bound on the target risk.
\par From now on, we define $\p \coloneqq (\Scaly[y=1],\cdots, \Scaly[y=k]) \in \Delta_K$, the vector of class proportions for the source domain. Also, we will refer to localization considered in \Cref{def:localizedSpace} and \ref{def:localizedSpacePerClass} respectively as \emph{global} and \emph{per-class} localization. The two notions are linked as stated by the following proposition.
\begin{restatable}{proposition}{propFepsRealVectorInclusion}\label{prop:FepsRealVectorInclusion}
    If $\epsbf = (\epsilon_1,\dots,\epsilon_K)^T \geq 0$\footnote{We write $\u\geq 0$ for $\u \in \Rset_+^p$.}, then
        $\Fbb_\epsbf \subseteq \Fbb_{\p^T\epsbf}$. Conversely, if $\etabf = (\eta_1,\cdots,\eta_K) \in \{\Rset_+ \cup \{\infty\} \}^K$ with $\eta_k = \cfrac{\epsilon}{\Scaly[y=k]}$ if $\Scaly[y=k]>0$ and $\eta_k = \infty$ otherwise. Then 
        $\Fbb_\epsilon \subseteq \Fbb_{\etabf}$.
\end{restatable}
An interpretation for the first result of \Cref{prop:FepsRealVectorInclusion} is that a hypothesis having a risk at  most $\epsilon$ per class will have a global risk at most $\epsilon$. For the second one, it shows that we can obtain localization per class at the price of losing in accuracy by dividing it by the class proportions. For example, on a class-balanced source domain, having an $l-$risk of at most $\epsilon>0$ implies having a risk of at most $K\epsilon$ per class.

\par Of course, the reasoning we used to link global localization and $\beta-$admissible distances through Propositions \ref{prop:BoundLocalizedIMD},\ref{prop:IMDDuality} and Corollary \ref{cor:BoundLocalizedIMDBeta} can be performed again for the per-class discrepancy, which leads to the following proposition.
\begin{restatable}{proposition}{propBoundLocalizedIMDBetaPerClass}\label{prop:BoundLocalizedIMDBetaPerClass}
	Let $\Hbb$ be a hypothesis space, $g$ a scoring function, and $l$ a loss function verifying the triangle inequality. Assume that $l(\Hbb,\Hbb)\subseteq \Fbb$, where $\Fbb$ is a set of bounded nonnegative functions. Also, for any $\rbf = (r_1, \cdots, r_K) \geq 0$, we consider the localized hypothesis space $
			\Hbb^{\rbf} \coloneqq \{h \in \Hbb; \risk{\Scal_{\Xcal|k}}{h} \leq r_k\ \forall 1\leq k \leq K\}$.
		\par Then for any $\rbf_1, \rbf_2, \betabf \geq 0$, we have 
		\begin{equation}
			\risk{\Tcal}{g} \leq \urisk{\Hbb^{\rbf_1}}{g} + \IMD_\Fbb\left(\Tcalx, \Scalx + \sum_{k=1}^K \beta_k\Scal_{\Xcal|k}\right) +\betabf^T(\rbf_1 + \rbf_2) + \inf_{h \in \Hbb^{\rbf_2}}\risk{\Tcal}{h} + \risk{\Scal}{h}.
		\end{equation}
\end{restatable}

\par The latter result is a generalization of Corollary \ref{cor:BoundLocalizedIMDBeta}. Indeed, for $\rbf_i = r_i\ones$, $i \in \{1,2\}$,\footnote{$\ones$ is the vector with all coordinates equal to 1.} and $\betabf = \beta\p$, where $\beta>0$ is a fixed parameter, we recover the global localization case. However, it requires $K$ parameters to fix (the components of $\betabf$) instead of one scalar $\beta\geq 0$, which is impractical. One can overcome this drawback by noticing the following: The choice of $\beta \p$ can be thought of as a particular splitting of $\beta$'s value across different classes. This naturally hints towards a better splitting. We formalize this observation in the following corollary.
\begin{restatable}{corollary}{corBoundLocalizedIMDBetaPerClassSplitting}\label{cor:BoundLocalizedIMDBetaPerClassSplitting}
	With the assumptions of \Cref{prop:BoundLocalizedIMDBetaPerClass}, let $r_1,r_2,\beta\geq 0$. Besides, let $\rbf_1 = \rbf_2 = r\ones$. Then
	\begin{equation}
		\risk{\Tcal}{g} \leq \urisk{\Hbb^{\rbf_1}}{g} + \min_{\substack{\betabf \geq 0\\\ones^T\betabf \leq \beta}}\quad\IMD_{\Fbb}\left(\Tcalx, \Scalx + \sum_{k=1}^K \beta_k \Scal_{\Xcal|k}\right) + \beta(r_1 + r_2) + \inf_{h \in \Hbb^{\rbf_2}}\risk{\Tcal}{h} + \risk{\Scal}{h}.
	\end{equation}
	Moreover, the inequality in the constraint $\beta\ones\leq 1$ can be replaced by an equality. 
\end{restatable}
Corollary \ref{cor:BoundLocalizedIMDBetaPerClassSplitting} shows the way to make the most of the per-class localization, while a priori fixing only one parameter $\beta$. Such an optimal splitting can be interpreted as a relaxed class reweighting. Indeed, denoting the $K-$dimensional probability simplex by $\Delta_K$, the minimization problem in Corollary \ref{cor:BoundLocalizedIMDBetaPerClassSplitting} is equivalent to 
\begin{equation}\label{pb:beta_relax_reweighting}
	\min_{\substack{\tilde \p \in \Delta_K\\ (1+\beta)\tilde \p \geq \p}}\quad\IMD_{\Fbb}\left(\Tcalx, (1+\beta)\sum_{k=1}^K \tilde p_k \Scal_{\Xcal|k}\right),\\
\end{equation}
which is a combination of $\beta-$relaxation \citep{wu2019domain} and re-weighting approaches \citep{redko2019optimal,tachet2020domain,rakotomamonjy2021optimal,kirchmeyer2022mapping}, although the reweighting is taken over a proper subset of the $K-$ dimensional probability simplex. However, to the best of our knowledge, our result is the first to involve a minimum over source class weights in a DA bound, in contrast with the previously mentioned contributions.

\par We now prove that for label shift, a special case considered in the DA literature \citep{zhang_domain_2013,lipton2018detecting,redko2019optimal} in which the distribution mismatch is due to a shift in the label marginals, the per-class localization is more beneficial than the global one.
\begin{restatable}[Case of label shift]{proposition}{propTargetShift}\label{prop:TargetShift}
	Assume that the source and target distributions verify $\Tcalxy = \Scalxy = \Dcalxy$ for all $y \in \Ycal$. Let $\{q_k\}_{k=1}^K$ denote the target class proportions. Then
	\begin{enumerate}
		\item If $\min_{1\leq k\leq K} p_k >0$ and $\beta \geq \max_{1\leq k\leq K} \left(\frac{q_k}{p_k} - 1\right)_+$, then $\IMD_\Fbb(\Tcalx, (1+\beta)\Scalx) = 0$.
		\item If $\beta_k \geq (q_k - p_k)_+\quad\forall 1\leq k\leq K$, then $\IMD_\Fbb(\Tcalx, \Scalx + \sum_{k=1}^K\beta_k\Scal_{\Xcal|k}) = 0$.
	\end{enumerate}
	Moreover, if $\Fbb$ is rich enough, and if there exists a family $\{B_l\}_{l=1}^K$ of subsets of $\Xcal$ such that $\Dcal_{\Xcal|k}(B_l)>0$ if $k=l$ and $\Dcal_{\Xcal|k}(B_l)=0$ otherwise, then the converses of the two previous statements holds.
\end{restatable}
Statements 1 and 2 of \Cref{prop:TargetShift} show the existence of some $\beta\geq 0$ or $\betabf \geq 0$ in the per-class localization case such that the relaxed IMD between $\Tcalx$ and $\Scalx$ is null. However, in the first case, $\beta$ grows with the ratio $\frac{q_k}{p_k}$ which can go arbitrarily large depending on the class proportions on the source domain, whereas in the case of per class localization, the lower bound on components $\beta_k$ is less prohibitive as it can at most be equal to 1 (since $\sum_{k=1}^K(q_k-p_k)_+ \leq 1$). Hence, the range of $\beta$ to test when implementing the $\beta-$splitting minimization problem of \Cref{cor:BoundLocalizedIMDBetaPerClassSplitting} is bounded. The converse statement relies on the capacity of $\Fbb$, and the existence of regions of $\Xcal$ having only one label. This is a relaxation of the cluster structure assumption \citep{tachet2020domain,kirchmeyer2022mapping} as the family $\{B_l\}_{l=1}^K$ does not need to form a partition of $\Xcal$. Whether the richness assumption can be relaxed (for example, for Lipschitz functions that define the Wasserstein distance) remains an open question.
\section{Some implications for specific choices of \texorpdfstring{$\Fbb$}{}}
\label{sec:special_cases}
In this section, we explore the consequences of some choices of $\Fbb$ on the definition of the relaxed IMD. We will link this choice to that of $\Hbb$, as we assume $l(\Hbb,\Hbb) \subseteq \Fbb$. Before continuing, we denote the support of a probability $\Dcalx$ over $\Xcal$ by $\supp\Dcalx$. We will successively consider $\Fbb$ as the space of all bounded measurable functions (infinite capacity), an arbitrary hypothesis space (typically with a finite VC dimension) and a space of $1-$Lipschitz functions.
\subsection{Bounded functions: revisiting total variation and reweighting}\label{sec:inf_capacity_case}
The following corollary revisits the total variation bound from \citet[Theorem 1]{ben-david_theory_2010}.
\begin{restatable}{proposition}{propDABoundTotalVariation}\label{prop:DABoundTotalVariation}
    Assume $l(\Hbb,\Hbb)$ is the space of all measurable functions bounded by $1$ over $\Xcal$. Denoting $t$ and $s$ the densities of $\Tcalx$ and $\Scalx$ with respect to some dominating measure $\mu$, we have for any scoring function $g$
    \begin{align}
            \risk{\Tcal}{g} &\leq \risk{\Tcal}{g,h_g} + \int \min(s,t)l(h_g,f_s) + d_1(\Tcalx,\Scalx) + \int \min(s,t)l(f_\Scal,f_\Tcal)\label{eq:DABoundInfCapacity}\\
            \risk{\Tcal}{g} &\leq \risk{\Tcal}{g,f_s} + 1-\Tcalx(\supp\Scalx) + \risk{\Tcal}{f_\Scal}\label{eq:DABoundInfCapacityZero}
    \end{align}
    where $d_1$ denotes the total variation distance.
\end{restatable}
The first result \eqref{eq:DABoundInfCapacity} of \Cref{prop:DABoundTotalVariation} is a refinement of a result from \citet{ben-david_theory_2010}. Apart from the total variation, the remaining terms are smaller due to point 1 of \Cref{prop:sourceGuidedUncertainty} and the $\min(s,t)$ term. The latter captures the unimportance of the covariate shift assumption outside of the overlap of the support of marginals (\ie, one can always extend $f_\Scal$ to be equal to $f_\Tcal$ on $\supp\Tcalx \setminus \supp\Scalx$, and \emph{vice-versa} for $f_\Tcal$). In the extreme case where both supports are disjoint, the $\min(s,t)$ term is null, letting only the total variation distance and the $\risk{\Tcalx}{g,h_g}$ the uncertainty of $g$ in its predictions (which disappears if $g$ is a classifier). Bound \eqref{eq:DABoundInfCapacity} is also linked to \citet{ben-david_hardness_2012}, specifically to their Algorithm $\mathcal A$ that considers an intersection between the source and target domains in terms of boxes of a grid defined over the $d-$dimensional unit cube. Our bound points out to the sufficiency of minimizing the source risk over the intersection of supports of both domains encoded by the $\min(s,t)$ term.
For the second bound \eqref{eq:DABoundInfCapacityZero}, it corresponds to the extreme localization case from \Cref{prop:BoundLocalizedIMD} for $r_1 = r_2 = 0$ or $\rbf_1 = \rbf_2 = 0$. It reduces to the importance-weighted risk on the source domain when the source support contains the target's, and when the labeling functions are the same, \ie in the classic covariate shift setting \citep{shimodaira_covariate_2000,huang2007correcting,sugiyama_covariate_2007,cortes_learning_2010}. This latter result, which only requires the inclusion of the support of the target in the source's, will have analogues that we will present in the next two sections.
\subsection{Hypothesis symmetric differences: revisiting $\HdH-$distances}\label{sec:HdH_case}
The seminal work of \citet{ben-david_theory_2010} proves a bound on the target risk involving the $\HdH$ distance, where $\HdH \coloneqq \{\x \mapsto [h_1(\x) \neq h_2(\x)]; h_1, h_2 \in \Hbb\}$ the space of disagreements between hypotheses in $\Hbb$.
Choosing $\Fbb = \HdH$, corresponding to setting $l$ to the $0-1$ loss, $\Fbb$ is not convex and thus only the upper bound from from \Cref{prop:IMDDuality} holds. The global localization in this case was the topic of \citet{zhang2020localized}. The next corollary addresses its upper bound from \Cref{prop:IMDDuality}.
\begin{restatable}{corollary}{corHdHIMD}
\label{cor:HdHIMD}
    For any $\beta\geq 0, \betabf \geq 0$, we have
    \begin{align}
         1- \IMD_{\HdH}(\Tcalx,(1+\beta)\Scalx) &= \inf _{f \in \HdH}\proba{\Tcalx}{f=0} + (1+\beta)\proba{\Scalx}{f=1}\label{eq:HdHBetaBinaryClassif}\\
         1 - \min_{\betabf^T\ones\leq \beta} \IMD_{\HdH}\left(\Tcalx, \Scalx+\sum_{k=1}^K \beta_k\Scal_{\Xcal|k}\right) &= \max_{\substack{\betabf \geq 0\\\betabf^T\ones\leq \beta}}\inf_{g \in \HdH} \proba{\Tcalx}{f=0} + \sum_{k=1}^K(p_k + \beta_k)\proba{\Scal_{\Xcal|k}}{f=1}\label{eq:HdHBetaBinaryClassifper-class}
    \end{align}
\end{restatable}
For $\beta=0$, the result of Equation \eqref{eq:HdHBetaBinaryClassif} coincides with the seminal result from \citet{ben-david_theory_2010} that links the $\HdH-$divergence and the binary classiciation problem of distinguishing the source from the target, and resulting in adversarial approaches \citep{ganin_domain-adversarial_2016}. Besides, it provides an additional justification to the deep DA approach proposed in \citet{zhang_bridging_2019}, in which a hyperparameter $\gamma$ multiplies the domain discriminator's risk on the source domain. As for result \eqref{eq:HdHBetaBinaryClassifper-class}, although it is not convex in $h$, it can be approximated in practice using convex surrogate loss functions, thus leading to a convex-concave for which the minimax theorem \citep{sion1958general} holds. Hence, it can be formulated as a binary classification that is robust to the choice of the $\betabf$ vector. A more rigorous link with the theory in this case remains an open direction linked with variational representations of $f-$divegences \citep{keziou2003dual,nguyen2009surrogate,reid2011information}.
\par The next proposition concerns the extreme case of localization with $r_1 = r_2 =0$ (or $\rbf_1 = \rbf_2 = 0$).
\begin{restatable}{proposition}{propHSupport}\label{prop:HSupport}
	Given a hypothesis space $\Hbb$, let $\Hbb\cdot\Hbb \coloneqq \{\x \mapsto [h_1(\x) = h_2(\x)]; h_1, h_2 \in \Hbb\}$, and define
 the $\Hbb\cdot\Hbb-$support of $\Scalx$ as 
	\begin{equation}
		\supp_{\Hbb\cdot\Hbb}\Scalx \coloneqq \bigcap_{\substack{h_1,h_2 \in \Hbb \\\proba{\Scalx}{h_1 = h_2}=1}} \{h_1=h_2\}.
	\end{equation}
	Then
	\begin{equation}\label{eq:boundOnIMDHdHZeroAsHSupport}
		\IMD_{\HdH_0}(\Tcalx,\Scalx) \leq 1-\Tcalx(\supp_{\Hbb\cdot\Hbb}\Scalx).
	\end{equation}
\end{restatable}
The inequality \eqref{eq:boundOnIMDHdHZeroAsHSupport} is a weakening of the result from bound \eqref{eq:DABoundInfCapacityZero}, in the same way that the $\HdH-$divergence generalizes the bound relying on the total variation by restricting the family of sets defining the supremum \citep{ben-david_theory_2010}. Indeed, while the support of $\Scalx$ is the intersection of all the closed sets with $\Scalx-$probability equal to $1$, the $\Hbb\cdot\Hbb-$support restricts the family of sets defining the intersection. For example, if the preimages of $1$ by hypotheses of $\Hbb\cdot\Hbb$ are convex sets, then $\supp_{\HdH}(\Scalx)$ is necessarily convex, while $\supp\Scalx$ does not need to be convex. Whether such a restriction improves the estimation of the $\Hbb\cdot\Hbb-$support over that of the classic support is a possible future research direction.
The previous analysis provides more justification for the DA method introduced in \citet{saito_maximum_2017}. Indeed, their theoretical justification involves approximating the $\HdH-$divergence as $\sup_{h_1,h_2}{h_1\neq h_2}$ when $h_1$ and $h_2$ are restricted to have correct predictions on $\Scal$, without examining the implications of this restriction on the rest of the $\HdH-$based DA bound from \citet{ben-david_theory_2010}. Our results from Proposition \ref{prop:BoundLocalizedIMD} and \ref{prop:BoundLocalizedIMDBetaPerClass} overcome this limitation. Moreover, by \Cref{prop:HSupport}, we show that their algorithm moves the target data to the structured support of the source, not the classic one.\\
\subsection{Lipschitz functions: revisiting Optimal Transport for Domain Adaptation}\label{sec:ot_case}
\par One of the most common assumptions about labeling functions is their Lipschitzness, as it represents an inductive bias allowing to propagate the label of an instance to its neighbors. Lipschitzness of labeling functions and large margin separation between classes representing high density clusters are two equivalent notions as noted in \citep{ben-david_domain_2014}. Relying on the Lipschitzness for the considered hypothesis space, several works in the literature establish bounds on the target risk involving the Wasserstein distance \citep{redko_theoretical_2017,courty_joint_2017,shen_wasserstein_2018,rakotomamonjy2021optimal,le2021lamda,kirchmeyer2022mapping}.
We now examine the implications of localization on this quantity. The case of global localization, \ie $\IMD_\Fbb(\Tcalx, (1+\beta)\Scalx)$ is exactly the partial optimal transport problem \citep{benamou2003numerical,caffarelli2010free,figalli2010optimal} introduced in \citet{wu2019domain}. As for the per-class localization, we provide its analogue in the following Proposition.
\begin{restatable}[Optimal transport with per-class localization]{proposition}{propOptimalTransportLocalized}\label{prop:optimalTransportLocalized}
    Assume $(\Xcal,d)$ and $(\Ycal, l)$ are metric spaces, and that the hypotheses in $\Hbb$ are $\frac{1}{2}-$Lipschitz. Then $l(\Hbb,\Hbb) \subseteq \Fbb$ where $\Fbb$ is the space of nonnegative $1-$Lipschitz functions over $\Xcal$. And for any $\betabf \geq 0$, the value of $\IMD_{\Fbb}\left(\Tcalx, \Scalx + \sum_{k=1}^K\beta_k\Scal_{\Xcal|k}\right)$ can be computed by solving the following transport problem
    \begin{equation}\label{pb:OTper-classLocalization}
        \begin{aligned}
            \inf_{\{\Pcal_k\}_{k=1}^K \subset \mathcal M_+(\Xcal,\Xcal)}&\quad \sum_{k=1}^K\esp{(\x_t, \x_s) \sim \Pcal_k}{d(\x_t, \x_s)}\\
            \text{s. t.}\quad& \pi_1 \# \sum_{k=1}^K \Pcal_k \geq \Tcalx\\
            &\pi_2 \# \Pcal_k \leq (p_k + \beta_k)\Scal_{\Xcal|k}\\
        \end{aligned}
    \end{equation}
	where $\pi_1:(\u,\v)\mapsto \u, \pi_2:(\u,\v)\mapsto \v$, and the first inequality constraint on $\Tcalx$ can be replaced by an equality.
\end{restatable}
The per-class formulation provides a new partial optimal transport problem in which the maximum mass received by the source depends on the different class-conditionals.
\par In the case $\epsilon = 0$, we can show that $\IMD_{\Fbb_0}(\Tcalx,\Scalx) = \esp{\Tcalx}{d(\x, \supp \Scalx)}$, hence providing a first theoretical justification for the use of the \emph{Symmetric Support Divergence (SSD)} \citep{tong2022adversarial} in domain adaptation via incorporating it in a bound on the target risk. Indeed, the SSD is an upper bound on $\esp{\Tcalx}{d(\x, \supp \Scalx)}$.
\section{Some empirical illustrations}
\label{sec:experiments}
\subsection{Experiment description}
\paragraph{Toy dataset generation:} For the source domain, we generate $n_s = 300$ points from a mixture of 2-dimensional $K$ Gaussians, \ie $\Scalx = \sum_{k=1}^K p_k \mathcal{N}(\mubf_k, \sigma\I_2)$ and each component corresponds to a class. We obtain the centers $\mubf_k$ by rotating the vector $(0,1)$ by an angle $\frac{2k\pi}{K}$. The label marginal distribution is $p_k \propto e^{\eta k}$, where $\eta>0$ captures the intensity of the imbalance. We generate the target domain by rotating the source by an angle $\theta$ around the origin. Its label proportions follow from sorting those of the source domain in the descending order, to accentuate the shift between $\Scaly$ and $\Tcaly$. Hence, $\theta = 0$ or $\theta>0$ respectively correspond to the label shift- and the generalized label shift cases.

\paragraph{Comparing per-class vs global relaxation:} We consider the special case of \Cref{sec:ot_case}. For a given $\beta>0$, we compute the corresponding transport plan for global \citep{wu2019domain} and per-class (minimization problem in Corollary \ref{cor:BoundLocalizedIMDBetaPerClassSplitting}), where computing the IMD is done by solving Problem \eqref{pb:OTper-classLocalization}. To assess whether the obtained transport plan respects the class information, we compute the accuracy of a labeling that results from the propagation of the source labels as proposed in \citet{redko2019optimal}. More precisely, given a transport plan $\P \in \Rset^{n_t\times n_s}$ and the source labels matrix $\Ybf_s \in \Rset^{n_s \times K}$, we estimate the label of target point $\x_t$ as $\hat y_t = \argmax(\P\Ybf_s)$: the label of a target point is a majority vote of the labels of the source instances to which its mass is transported. After computing the difference in accuracy between per-class and global relaxation over 50 draws of source and target data, we report the median, the maximum, and the minimum of these draws. When not mentioned on the figures, different parameters are fixed as $K= 3$, $\beta=0.5$, and $\eta=1$.
\subsection{Results}
\subsubsection*{Label shift case}
For $\theta = 0\degree$, \Cref{fig:illustrationLabelShift} illustrates \Cref{prop:TargetShift}. Recall that the proposition guarantees the existence of some $\beta \in [0,1]$ that makes the IMD null for the per-class relaxation, whereas it is not the case for the global case \citep{wu2019domain}. In fact, solving the per-class relaxed optimal transport problem results in a transportation plan without cross-labeling, whereas this latter issue persists with global localization. \Cref{fig:accuracyVsBetaLabelShift} also shows this advantage by examining the difference in accuracy between the per-class and global localization. In particular, it shows that the difference grows larger as the number of classes or the imbalance intensity $\eta$ grows for the same sample size.

\begin{figure}
	\begin{subfigure}{0.33\linewidth}
		\includegraphics[width=\linewidth]{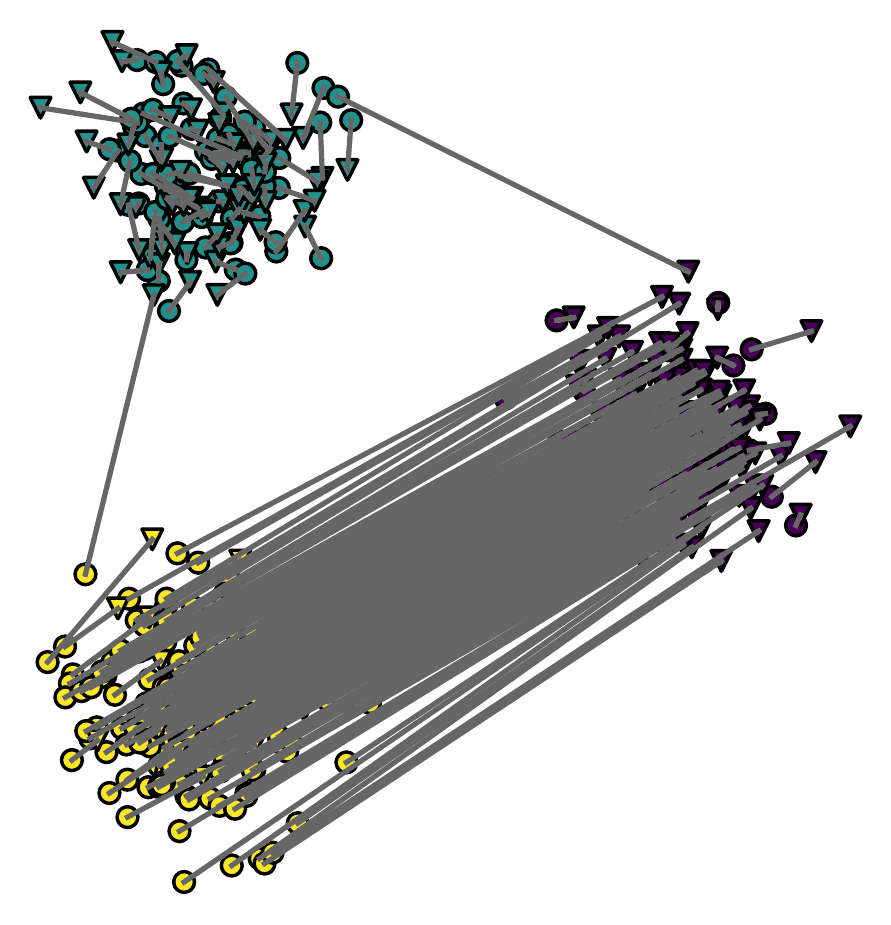}
		\caption{$\beta = 0$ (classic optimal transport)}
	\end{subfigure}
	\begin{subfigure}{0.33\linewidth}
		\includegraphics[width=\linewidth]{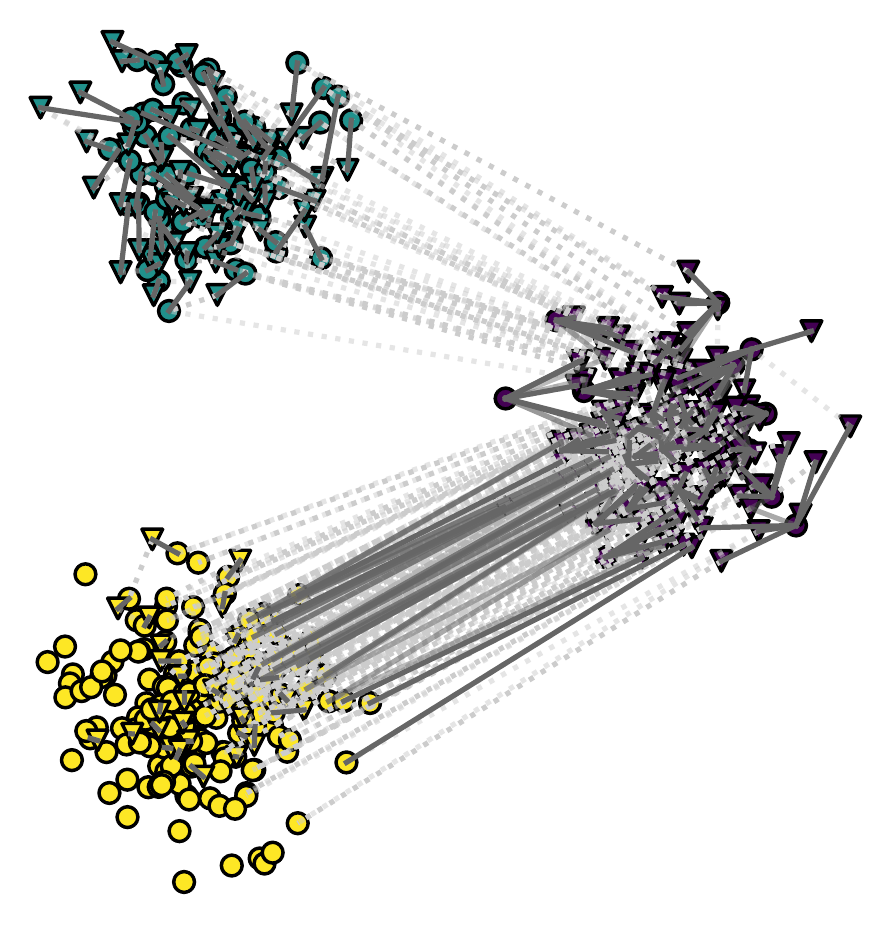}
		\caption{$\beta = 0.5$}
	\end{subfigure}
	\begin{subfigure}{0.33\linewidth}
		\includegraphics[width=\linewidth]{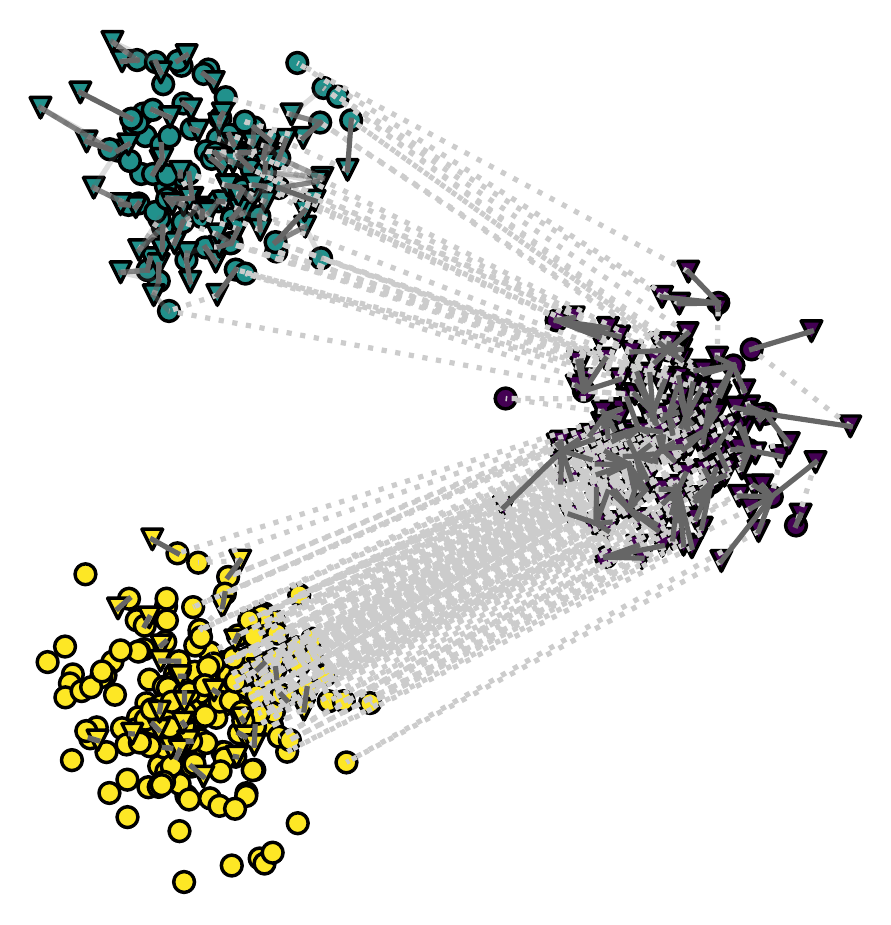}
		\caption{$\beta = 1$}
	\end{subfigure}
	\caption{Illustration of the transport plan obtained for $\theta= 0\degree$. Circles and triangles correspond to $\Scalx$ and $\Tcalx$ respectively. The transport plan is illustrated with continuous (resp. dotted) lines for per-class (resp. global) relaxation.}
	\label{fig:illustrationLabelShift}
\end{figure}

\begin{figure}[!ht]
	\centering
	\begin{subfigure}{0.45\linewidth}
		\includegraphics[width=\linewidth]{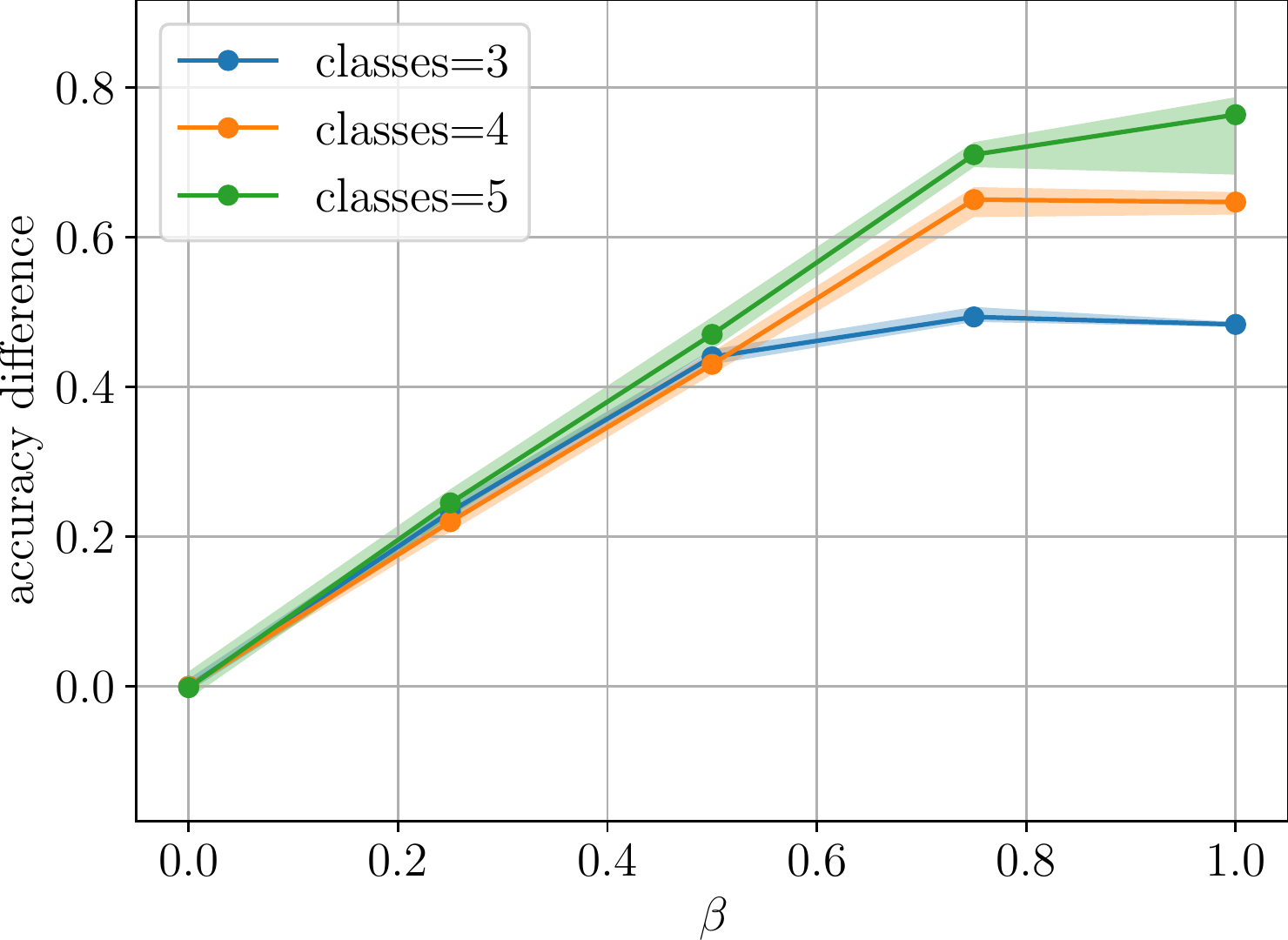}
		\caption{for different values of $K$}
	\end{subfigure}
	\begin{subfigure}{0.45\linewidth}
		\includegraphics[width=\linewidth]{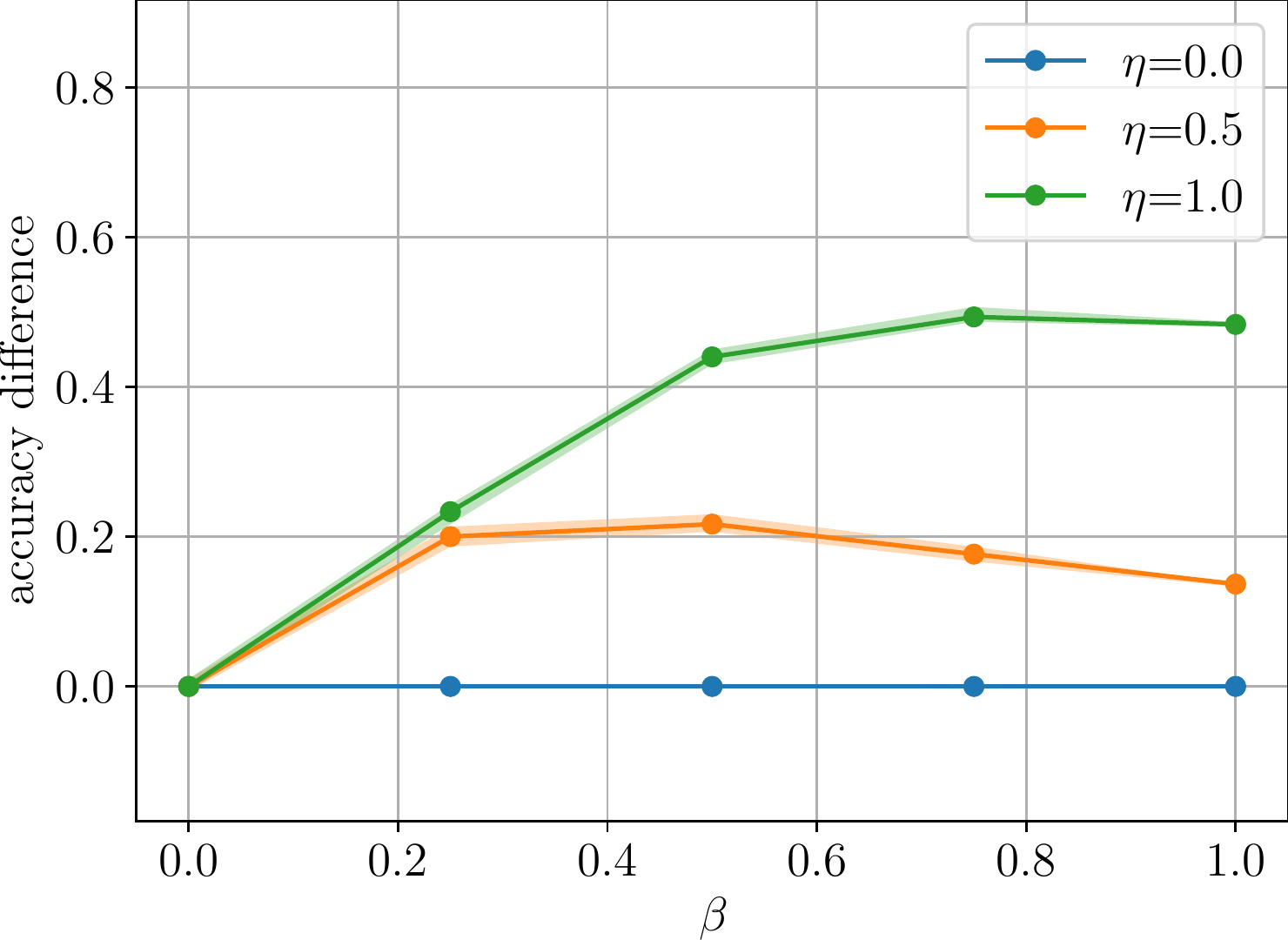}
		\caption{for different imbalance values}
	\end{subfigure}
	\caption{Difference of accuracy due to label propagation as a function of $\beta$.}
	\label{fig:accuracyVsBetaLabelShift}
\end{figure}
\subsubsection*{Generalized label shift case}
For $\theta = 30\degree$, we carry out the same experiments as in the label shift case. We do not have a theoretical guarantee on the value of $\beta$ for the per-class relaxation, but we notice a similar phenomenon when $\beta$ is allowed to grow greater than $1$, as illustrated in \Cref{fig:illustrationAngle30}: less connections between instances of different label are present as $\beta$ increases. However, on \Cref{fig:accuracyVsBetaAngle30}, we no longer have the same trend when varying the class numbers, whereas a similar trend is observed for different imbalance intensities $\eta$. The change in trend observed in \Cref{fig:accuracy_diff_vs_beta_for_n_classes_30} can be explained as follows: When the number of classes grows, instances of different classes get closer to each other. As a result, the difficulty of the adaptation outweighs the benefit of per-class localization.
\begin{figure}[!ht] 
	\begin{subfigure}{0.32\linewidth}
		\includegraphics[width=\linewidth]{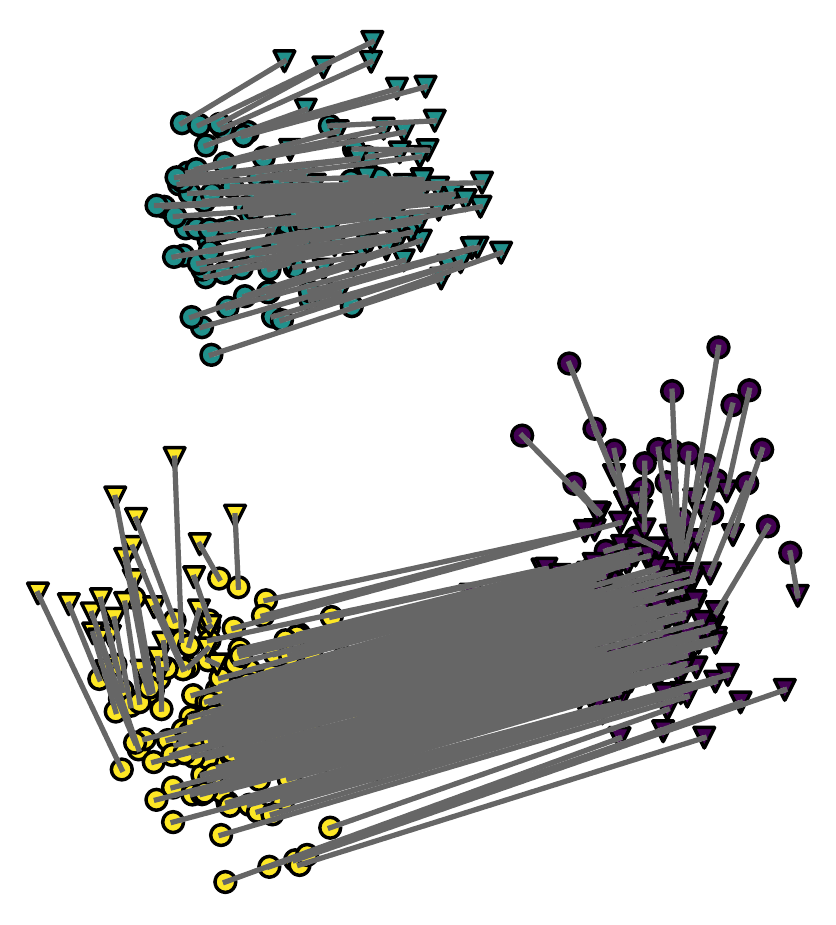}
		\caption{$\beta = 0$ (classic optimal transport)}
	\end{subfigure}
	\begin{subfigure}{0.32\linewidth}
		\includegraphics[width=\linewidth]{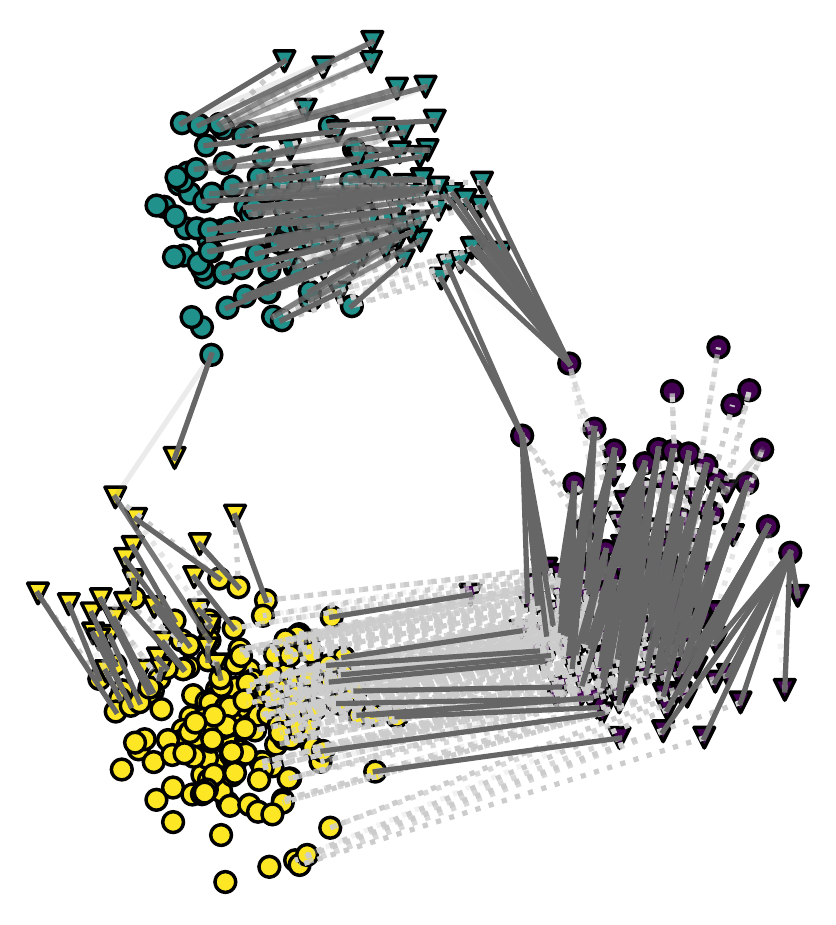}
		\caption{$\beta = 0.75$}
	\end{subfigure}
	\begin{subfigure}{0.32\linewidth}
		\includegraphics[width=\linewidth]{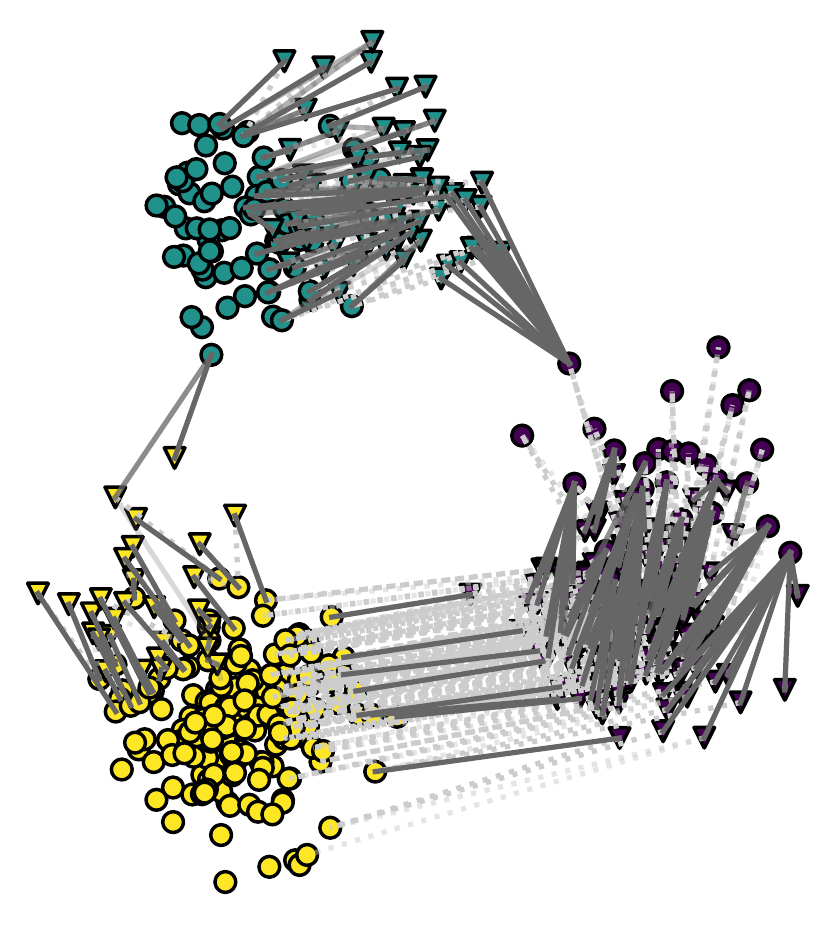}
		\caption{$\beta = 1.5$}
	\end{subfigure}
	\caption{Illustration of the transport plan obtained by solving the problem with global vs per-class relaxation. Circles and triangles correspond to the source and target domains respectively.}
	\label{fig:illustrationAngle30}
\end{figure}

\begin{figure}[!ht]
	\centering
	\begin{subfigure}{0.45\linewidth}
		\includegraphics[width=\linewidth]{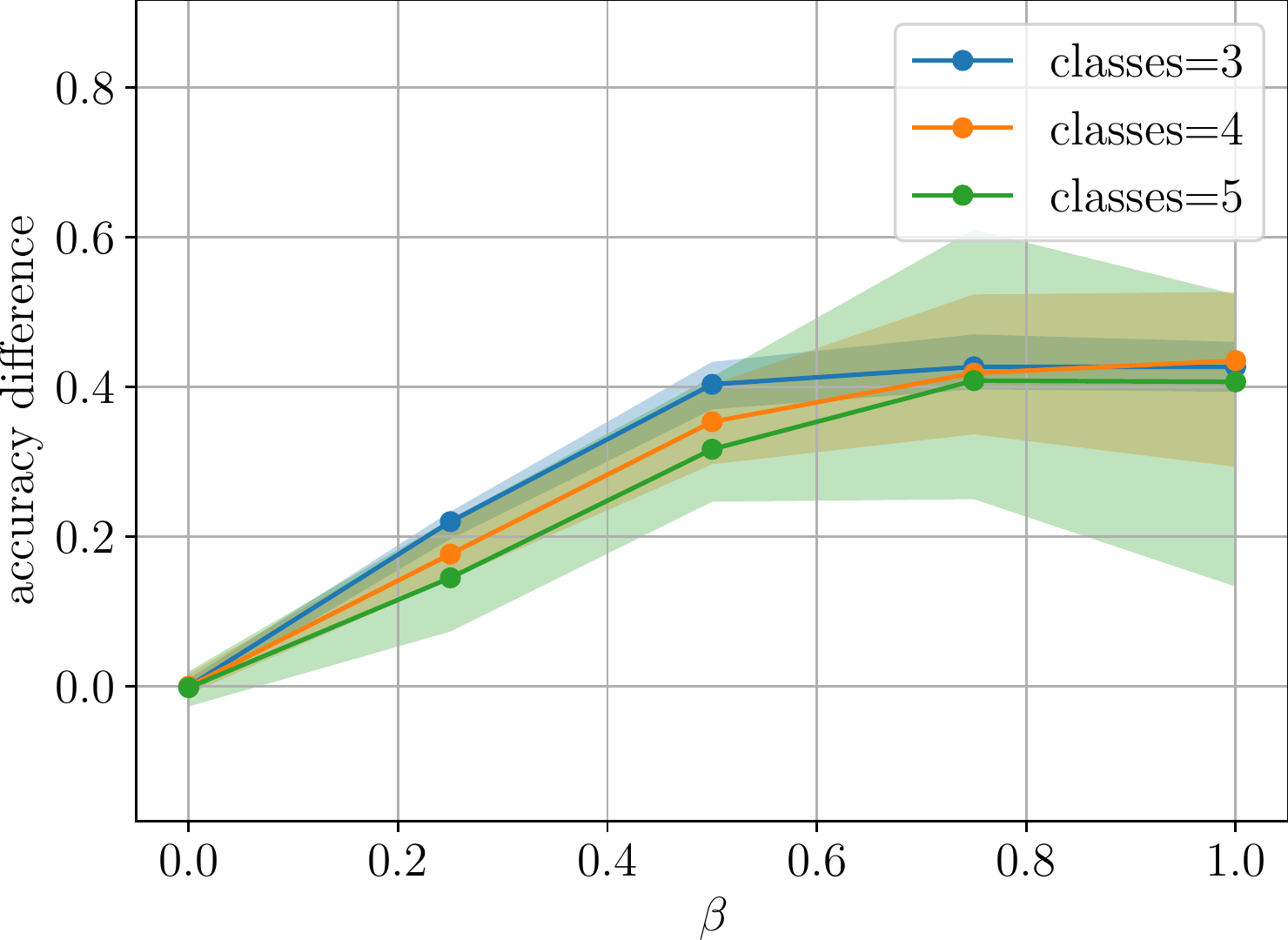}
		\caption{for different values of $K$}
		\label{fig:accuracy_diff_vs_beta_for_n_classes_30}
	\end{subfigure}
	\begin{subfigure}{0.45\linewidth}
		\includegraphics[width=\linewidth]{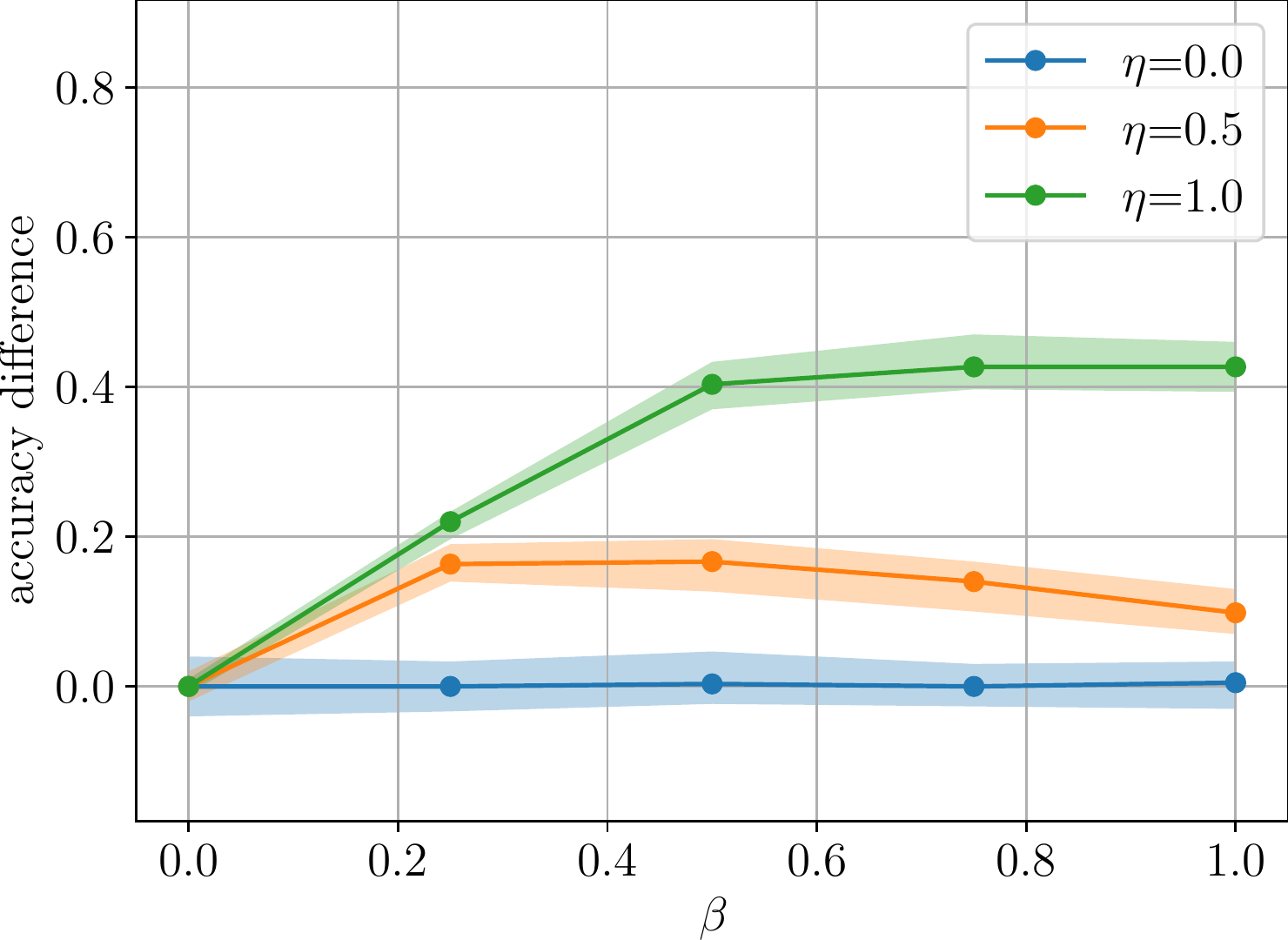}
		\caption{for different imbalance values}
	\end{subfigure}
	\caption{Difference in accuracy due to label propagation as a function of $\beta$.}
	\label{fig:accuracyVsBetaAngle30}
\end{figure}

\section{Related Work}
\paragraph{Confident predictions} The idea of promoting the confidence of a classifier in its predictions appears in the semi-supervised learning literature, where the goal is to encourage the cluster assumption. It implies that no high density region is crossed by decision boundaries \citep{grandvalet2005semi}. Some early domain adaptation approaches relied purely on this assumption without aligning the domains \citep{bruzzone_domain_2010}. Newer methods encourage it via conditional entropy minimization along with domain alignment and source performance optimization \citep{shu_dirt-t_2018,kirchmeyer2022mapping,tong2022adversarial}, or rely on it to define pseudo-labels when the considered approach requires some supervision on the target domain \citep{kang_contrastive_2019}.  In \citet{ben-david_hardness_2012,ben-david_domain_2014}, the authors studied the cluster assumption for its impact on the sample complexity of some DA algorithms, where it concerns the labeling function, not the learned classifier as in our case. More recently, the contributions of \citet{germain_new_2016,morerio_correlation_2017} addressed this problem theoretically. The former leverages the PAC-Bayesian theory \citep{mcallester1999some,catoni2007pac} to prove a theoretical bound that is a combination of the source risk and an unsupervised target risk in a binary classification setting. The latter shows a relation between the alignment of covariance matrices \citep{sun_return_2016,sun_deep_2016} and conditional entropy minimization on the target domain. Our work, however, provides justification in a more general multi-class setting, and is an improvement over a general class of bounds.
\paragraph{Asymmetry in the relation between domains}
Relaxing the requirement of equality of the distribution marginals is not new and has been considered in the definition of the $\HdH-$distance \citep{ben-david_theory_2010} and the $l-$discrepancy \citep{mansour_domain_2009}. In both cases, the restriction of the supremum over the considered hypothesis space makes the divergence null even when distributions are not equal. In contrast, the interest in asymmetric divergences is more recent \citep{zhang_bridging_2019,wu2019domain,zhang2020localized,kpotufe2021marginal}. In particular, \citet{zhang2020localized,hanneke2019value,kpotufe2021marginal} advocated the interest of asymmetry to capture the easiness of adaptation depending on its direction, and showed the benefit of asymmetry in reducing the sample complexity at the source level. Although we did not study sample complexity, our IMD notion can be considered complementary to these approaches. In fact, it relates to several previously considered IPM's that led to the implementations of DA algorithms.

\section{Conclusion and future perspectives}
In this work, we provided several refinements to domain adaptation theory on two main aspects. On the one hand, we highlighted the role of the certainty of the considered scoring function in its predictions on the target domain while being guided by the source labels. Our result in this regard spans over a large class of domain adaptation theoretical bounds. On the other hand, we connected two families of relaxations of divergences between probabilities and we extended them by utilizing the source domain's categorical information.
\par The future perspectives of this work are many. Indeed, whether a prior on target marginal class distribution will lead to other forms of divergence relaxation is an open direction. Also, other specializations of our divergence term such as considering universal kernels are to consider. And apart from IPM's, variational representations of $f-$divergences would link our analysis to adversarial deep learning methods. Finally, providing generalization rates in the same way as in \citet{zhang2020localized} would be informative about the role of localization depending on the choice of $\Fbb$ and whether it is performed per class.
\section{Acknowledgements} We thank Ievgen Redko for the fruitful discussions and the valuable feedback.

\newpage
\bibliography{./bib/DA,./bib/probas,./bib/other}

\begin{thebibliography}{90}
\providecommand{\natexlab}[1]{#1}
\providecommand{\url}[1]{\texttt{#1}}
\expandafter\ifx\csname urlstyle\endcsname\relax
  \providecommand{\doi}[1]{doi: #1}\else
  \providecommand{\doi}{doi: \begingroup \urlstyle{rm}\Url}\fi

\bibitem[Acuna et~al.(2021)Acuna, Zhang, Law, and Fidler]{acuna2021f}
D.~Acuna, G.~Zhang, M.~T. Law, and S.~Fidler.
\newblock f-domain adversarial learning: Theory and algorithms.
\newblock In \emph{Proceedings of the 38th International Conference on Machine
  Learning}, pages 66--75. PMLR, 2021.

\bibitem[Ben-David and Urner(2012)]{ben-david_hardness_2012}
S.~Ben-David and R.~Urner.
\newblock On the {Hardness} of {Domain} {Adaptation} and the {Utility} of
  {Unlabeled} {Target} {Samples}.
\newblock In \emph{Algorithmic {Learning} {Theory}}, Lecture {Notes} in
  {Computer} {Science}, pages 139--153, 2012.

\bibitem[Ben-David and Urner(2014)]{ben-david_domain_2014}
S.~Ben-David and R.~Urner.
\newblock Domain adaptation–can quantity compensate for quality?
\newblock \emph{Annals of Mathematics and Artificial Intelligence}, 70\penalty0
  (3):\penalty0 185--202, 2014.

\bibitem[Ben-David et~al.(2007)Ben-David, Blitzer, Crammer, Kulesza, and
  Pereira]{ben-david_analysis_2007}
S.~Ben-David, J.~Blitzer, K.~Crammer, A.~Kulesza, and F.~Pereira.
\newblock Analysis of {Representations} for {Domain} {Adaptation}.
\newblock In \emph{Advances in {Neural} {Information} {Processing} {Systems}
  19}, pages 137--144. 2007.

\bibitem[{Ben{-}David} et~al.(2010){Ben{-}David}, Blitzer, Crammer, Kulesza,
  Pereira, and Vaughan]{ben-david_theory_2010}
S.~{Ben{-}David}, J.~Blitzer, K.~Crammer, A.~Kulesza, F.~Pereira, and J.~W.
  Vaughan.
\newblock A theory of learning from different domains.
\newblock \emph{Mach. Learn.}, 79\penalty0 (1-2):\penalty0 151--175, 2010.

\bibitem[Benamou(2003)]{benamou2003numerical}
J.-D. Benamou.
\newblock Numerical resolution of an “unbalanced” mass transport problem.
\newblock \emph{ESAIM: Mathematical Modelling and Numerical Analysis},
  37\penalty0 (5):\penalty0 851--868, 2003.

\bibitem[Benamou et~al.(2015)Benamou, Carlier, Cuturi, Nenna, and
  Peyr{\'e}]{benamou2015iterative}
J.-D. Benamou, G.~Carlier, M.~Cuturi, L.~Nenna, and G.~Peyr{\'e}.
\newblock Iterative bregman projections for regularized transportation
  problems.
\newblock \emph{SIAM Journal on Scientific Computing}, 37\penalty0
  (2):\penalty0 A1111--A1138, 2015.

\bibitem[Blitzer et~al.(2007)Blitzer, Dredze, and
  Pereira]{blitzer_biographies_2007}
J.~Blitzer, M.~Dredze, and F.~Pereira.
\newblock Biographies, bollywood, boomboxes and blenders: {Domain} adaptation
  for sentiment classification.
\newblock In \emph{In {ACL}}, pages 187--205, 2007.

\bibitem[Boser et~al.(1992)Boser, Guyon, and Vapnik]{boser_training_1992}
B.~E. Boser, I.~M. Guyon, and V.~N. Vapnik.
\newblock A {Training} {Algorithm} for {Optimal} {Margin} {Classifiers}.
\newblock In \emph{Proceedings of the 5th {Annual} {ACM} {Workshop} on
  {Computational} {Learning} {Theory}}, pages 144--152, 1992.

\bibitem[Bruzzone and Marconcini(2010)]{bruzzone_domain_2010}
L.~Bruzzone and M.~Marconcini.
\newblock Domain {Adaptation} {Problems}: {A} {DASVM} {Classification}
  {Technique} and a {Circular} {Validation} {Strategy}.
\newblock \emph{IEEE Transactions on Pattern Analysis and Machine
  Intelligence}, 32\penalty0 (5):\penalty0 770--787, 2010.

\bibitem[Caffarelli and McCann(2010)]{caffarelli2010free}
L.~A. Caffarelli and R.~J. McCann.
\newblock Free boundaries in optimal transport and monge-ampere obstacle
  problems.
\newblock \emph{Annals of mathematics}, pages 673--730, 2010.

\bibitem[Catoni(2007)]{catoni2007pac}
O.~Catoni.
\newblock Pac-bayesian supervised classification: the thermodynamics of
  statistical learning.
\newblock \emph{arXiv preprint arXiv:0712.0248}, 2007.

\bibitem[Chapelle and Zien(2005)]{chapelle2005semi}
O.~Chapelle and A.~Zien.
\newblock Semi-supervised classification by low density separation.
\newblock In \emph{International workshop on artificial intelligence and
  statistics}, pages 57--64. PMLR, 2005.

\bibitem[Chizat et~al.(2018{\natexlab{a}})Chizat, Peyr{\'e}, Schmitzer, and
  Vialard]{chizat2018scaling}
L.~Chizat, G.~Peyr{\'e}, B.~Schmitzer, and F.-X. Vialard.
\newblock Scaling algorithms for unbalanced optimal transport problems.
\newblock \emph{Mathematics of Computation}, 87\penalty0 (314):\penalty0
  2563--2609, 2018{\natexlab{a}}.

\bibitem[Chizat et~al.(2018{\natexlab{b}})Chizat, Peyr{\'e}, Schmitzer, and
  Vialard]{chizat2018unbalanced}
L.~Chizat, G.~Peyr{\'e}, B.~Schmitzer, and F.-X. Vialard.
\newblock Unbalanced optimal transport: Dynamic and kantorovich formulations.
\newblock \emph{Journal of Functional Analysis}, 274\penalty0 (11):\penalty0
  3090--3123, 2018{\natexlab{b}}.

\bibitem[Collobert et~al.(2006)Collobert, Sinz, Weston, Bottou, and
  Joachims]{collobert2006large}
R.~Collobert, F.~Sinz, J.~Weston, L.~Bottou, and T.~Joachims.
\newblock Large scale transductive svms.
\newblock \emph{Journal of Machine Learning Research}, 7\penalty0 (8), 2006.

\bibitem[Cortes and Mohri(2011)]{cortes_domain_2011}
C.~Cortes and M.~Mohri.
\newblock Domain adaptation in regression.
\newblock In \emph{ALT}, 2011.

\bibitem[Cortes and Mohri(2014)]{cortes_domain_2014}
C.~Cortes and M.~Mohri.
\newblock Domain adaptation and sample bias correction theory and algorithm for
  regression.
\newblock \emph{Theoretical Computer Science}, 519:\penalty0 103--126, 2014.

\bibitem[Cortes and Vapnik(1995)]{cortes_support-vector_1995}
C.~Cortes and V.~Vapnik.
\newblock Support-{Vector} {Networks}.
\newblock \emph{Machine Learning}, 20\penalty0 (3):\penalty0 273--297, 1995.

\bibitem[Cortes et~al.(2010)Cortes, Mansour, and Mohri]{cortes_learning_2010}
C.~Cortes, Y.~Mansour, and M.~Mohri.
\newblock Learning bounds for importance weighting.
\newblock In \emph{Advances in Neural Information Processing Systems},
  volume~23, 2010.

\bibitem[Courty et~al.(2016)Courty, Flamary, Tuia, and
  Rakotomamonjy]{courty_optimal_2016}
N.~Courty, R.~Flamary, D.~Tuia, and A.~Rakotomamonjy.
\newblock {Optimal Transport for Domain Adaptation}.
\newblock \emph{{IEEE Transactions on Pattern Analysis and Machine
  Intelligence}}, 39\penalty0 (9):\penalty0 1853--1865, 2016.

\bibitem[Courty et~al.(2017)Courty, Flamary, Habrard, and
  Rakotomamonjy]{courty_joint_2017}
N.~Courty, R.~Flamary, A.~Habrard, and A.~Rakotomamonjy.
\newblock Joint distribution optimal transportation for domain adaptation.
\newblock In \emph{Advances in Neural Information Processing Systems}, pages
  3730--3739, 2017.

\bibitem[Csisz{\'a}r(1967)]{csiszar1967information}
I.~Csisz{\'a}r.
\newblock Information-type measures of difference of probability distributions
  and indirect observation.
\newblock \emph{studia scientiarum Mathematicarum Hungarica}, 2:\penalty0
  229--318, 1967.

\bibitem[Csurka et~al.(2017)]{csurka2017domain}
G.~Csurka et~al.
\newblock \emph{Domain adaptation in computer vision applications}.
\newblock Springer, 2017.

\bibitem[Daumé~III(2009)]{daume_iii_frustratingly_2009}
H.~Daumé~III.
\newblock Frustratingly {Easy} {Domain} {Adaptation}.
\newblock \emph{arXiv:0907.1815 [cs]}, 2009.

\bibitem[Erkan and Altun(2010)]{erkan2010semi}
A.~Erkan and Y.~Altun.
\newblock Semi-supervised learning via generalized maximum entropy.
\newblock In \emph{Proceedings of the Thirteenth International Conference on
  Artificial Intelligence and Statistics}, pages 209--216. JMLR Workshop and
  Conference Proceedings, 2010.

\bibitem[Fatras et~al.(2021)Fatras, S{\'e}journ{\'e}, Flamary, and
  Courty]{fatras2021unbalanced}
K.~Fatras, T.~S{\'e}journ{\'e}, R.~Flamary, and N.~Courty.
\newblock Unbalanced minibatch optimal transport; applications to domain
  adaptation.
\newblock In \emph{International Conference on Machine Learning}, pages
  3186--3197. PMLR, 2021.

\bibitem[Fernando et~al.(2013)Fernando, Habrard, Sebban, and
  Tuytelaars]{fernando_subspace_2014}
B.~Fernando, A.~Habrard, M.~Sebban, and T.~Tuytelaars.
\newblock {Unsupervised Visual Domain Adaptation Using Subspace Alignment}.
\newblock In \emph{{ICCV 2013}}, pages 2960--2967, Sydney, Australia, 2013.

\bibitem[Figalli(2010)]{figalli2010optimal}
A.~Figalli.
\newblock The optimal partial transport problem.
\newblock \emph{Archive for rational mechanics and analysis}, 195\penalty0
  (2):\penalty0 533--560, 2010.

\bibitem[Ganin et~al.(2016)Ganin, Ustinova, Ajakan, Germain, Larochelle,
  Laviolette, March, and Lempitsky]{ganin_domain-adversarial_2016}
Y.~Ganin, E.~Ustinova, H.~Ajakan, P.~Germain, H.~Larochelle, F.~Laviolette,
  M.~March, and V.~Lempitsky.
\newblock Domain-{Adversarial} {Training} of {Neural} {Networks}.
\newblock \emph{Journal of Machine Learning Research}, 17\penalty0
  (59):\penalty0 1--35, 2016.

\bibitem[Germain et~al.(2013)Germain, Habrard, Laviolette, and
  Morvant]{germain_pac-bayesian_2013}
P.~Germain, A.~Habrard, F.~Laviolette, and E.~Morvant.
\newblock A {PAC}-{Bayesian} {Approach} for {Domain} {Adaptation} with
  {Specialization} to {Linear} {Classifiers}.
\newblock In \emph{International {Conference} on {Machine} {Learning}}, pages
  738--746, 2013.

\bibitem[Germain et~al.(2016)Germain, Habrard, Laviolette, and
  Morvant]{germain_new_2016}
P.~Germain, A.~Habrard, F.~Laviolette, and E.~Morvant.
\newblock A {New} {PAC}-{Bayesian} {Perspective} on {Domain} {Adaptation}.
\newblock In \emph{International {Conference} on {Machine} {Learning}}, pages
  859--868, 2016.

\bibitem[Gong et~al.(2013)Gong, Grauman, and Sha]{gong_connecting_2013}
B.~Gong, K.~Grauman, and F.~Sha.
\newblock Connecting the {Dots} with {Landmarks}: {Discriminatively} {Learning}
  {Domain}-{Invariant} {Features} for {Unsupervised} {Domain} {Adaptation}.
\newblock In \emph{International {Conference} on {Machine} {Learning}}, pages
  222--230, 2013.

\bibitem[Goodfellow et~al.(2016)Goodfellow, Bengio, and
  Courville]{goodfellow2016deep}
I.~Goodfellow, Y.~Bengio, and A.~Courville.
\newblock \emph{Deep learning}.
\newblock MIT press, 2016.

\bibitem[Grandvalet et~al.(2005)Grandvalet, Bengio, et~al.]{grandvalet2005semi}
Y.~Grandvalet, Y.~Bengio, et~al.
\newblock Semi-supervised learning by entropy minimization.
\newblock \emph{CAP}, 367:\penalty0 281--296, 2005.

\bibitem[Hanneke and Kpotufe(2019)]{hanneke2019value}
S.~Hanneke and S.~Kpotufe.
\newblock On the value of target data in transfer learning.
\newblock 32, 2019.

\bibitem[Huang et~al.(2007)Huang, Gretton, Borgwardt, Sch{\"o}lkopf, and
  Smola]{huang2007correcting}
J.~Huang, A.~Gretton, K.~Borgwardt, B.~Sch{\"o}lkopf, and A.~J. Smola.
\newblock Correcting sample selection bias by unlabeled data.
\newblock In \emph{Advances in neural information processing systems}, pages
  601--608, 2007.

\bibitem[Johansson et~al.(2019)Johansson, Sontag, and
  Ranganath]{johansson_support_2019}
F.~D. Johansson, D.~Sontag, and R.~Ranganath.
\newblock Support and {Invertibility} in {Domain}-{Invariant}
  {Representations}.
\newblock In \emph{The 22nd {International} {Conference} on {Artificial}
  {Intelligence} and {Statistics}}, pages 527--536, 2019.

\bibitem[Kang et~al.(2019)Kang, Jiang, Yang, and
  Hauptmann]{kang_contrastive_2019}
G.~Kang, L.~Jiang, Y.~Yang, and A.~G. Hauptmann.
\newblock Contrastive {Adaptation} {Network} for {Unsupervised} {Domain}
  {Adaptation}.
\newblock In \emph{2019 {IEEE}/{CVF} {Conference} on {Computer} {Vision} and
  {Pattern} {Recognition} ({CVPR})}, pages 4888--4897, 2019.

\bibitem[Keziou(2003)]{keziou2003dual}
A.~Keziou.
\newblock Dual representation of $\varphi$-divergences and applications.
\newblock \emph{Comptes rendus math{\'e}matique}, 336\penalty0 (10):\penalty0
  857--862, 2003.

\bibitem[Kirchmeyer et~al.(2022)Kirchmeyer, Rakotomamonjy, de~Bezenac, and
  patrick gallinari]{kirchmeyer2022mapping}
M.~Kirchmeyer, A.~Rakotomamonjy, E.~de~Bezenac, and patrick gallinari.
\newblock Mapping conditional distributions for domain adaptation under
  generalized target shift.
\newblock In \emph{International Conference on Learning Representations}, 2022.
\newblock URL \url{https://openreview.net/forum?id=sPfB2PI87BZ}.

\bibitem[Kouw and Loog(2019)]{kouw_review_2019}
W.~M. Kouw and M.~Loog.
\newblock A review of single-source unsupervised domain adaptation.
\newblock \emph{arXiv:1901.05335 [cs, stat]}, 2019.

\bibitem[Kpotufe and Martinet(2021)]{kpotufe2021marginal}
S.~Kpotufe and G.~Martinet.
\newblock Marginal singularity and the benefits of labels in covariate-shift.
\newblock \emph{The Annals of Statistics}, 49\penalty0 (6):\penalty0
  3299--3323, 2021.

\bibitem[Le et~al.(2021)Le, Nguyen, Ho, Bui, and Phung]{le2021lamda}
T.~Le, T.~Nguyen, N.~Ho, H.~Bui, and D.~Phung.
\newblock Lamda: Label matching deep domain adaptation.
\newblock In \emph{International Conference on Machine Learning}, pages
  6043--6054. PMLR, 2021.

\bibitem[Liang et~al.(2021)Liang, Hu, Wang, He, and Feng]{liang2021source}
J.~Liang, D.~Hu, Y.~Wang, R.~He, and J.~Feng.
\newblock Source data-absent unsupervised domain adaptation through hypothesis
  transfer and labeling transfer.
\newblock \emph{IEEE Transactions on Pattern Analysis and Machine
  Intelligence}, 2021.

\bibitem[Liero et~al.(2018)Liero, Mielke, and Savaré]{liero_optimal_2018}
M.~Liero, A.~Mielke, and G.~Savaré.
\newblock Optimal entropy-transport problems and a new hellinger–kantorovich
  distance between positive measures.
\newblock \emph{Inventiones mathematicae}, 211\penalty0 (3):\penalty0
  969--1117, 2018.

\bibitem[Lipton et~al.(2018)Lipton, Wang, and Smola]{lipton2018detecting}
Z.~Lipton, Y.-X. Wang, and A.~Smola.
\newblock Detecting and correcting for label shift with black box predictors.
\newblock In \emph{International conference on machine learning}, pages
  3122--3130. PMLR, 2018.

\bibitem[Long et~al.(2018)Long, CAO, Wang, and Jordan]{long2018conditional}
M.~Long, Z.~CAO, J.~Wang, and M.~I. Jordan.
\newblock Conditional adversarial domain adaptation.
\newblock In S.~Bengio, H.~Wallach, H.~Larochelle, K.~Grauman, N.~Cesa-Bianchi,
  and R.~Garnett, editors, \emph{Advances in Neural Information Processing
  Systems}, volume~31. Curran Associates, Inc., 2018.

\bibitem[Luenberger(1997)]{luenberger1997optimization}
D.~G. Luenberger.
\newblock \emph{Optimization by vector space methods}.
\newblock John Wiley \& Sons, 1997.

\bibitem[Mansour et~al.(2009)Mansour, Mohri, and
  Rostamizadeh]{mansour_domain_2009}
Y.~Mansour, M.~Mohri, and A.~Rostamizadeh.
\newblock Domain adaptation: Learning bounds and algorithms.
\newblock In \emph{{COLT} 2009 - The 22nd Conference on Learning Theory,
  Montreal, Quebec, Canada, June 18-21, 2009}, 2009.

\bibitem[McAllester(1999)]{mcallester1999some}
D.~A. McAllester.
\newblock Some pac-bayesian theorems.
\newblock \emph{Machine Learning}, 37\penalty0 (3):\penalty0 355--363, 1999.

\bibitem[Mehra et~al.(2021)Mehra, Kailkhura, Chen, and
  Hamm]{mehra2021understanding}
A.~Mehra, B.~Kailkhura, P.-Y. Chen, and J.~Hamm.
\newblock Understanding the limits of unsupervised domain adaptation via data
  poisoning.
\newblock \emph{arXiv preprint arXiv:2107.03919}, 2021.

\bibitem[Morerio and Murino(2017)]{morerio_correlation_2017}
P.~Morerio and V.~Murino.
\newblock Correlation {Alignment} by {Riemannian} {Metric} for {Domain}
  {Adaptation}.
\newblock \emph{arXiv:1705.08180 [cs]}, 2017.

\bibitem[M\"{u}ller(1997)]{mueller1997integral}
A.~M\"{u}ller.
\newblock {Integral Probability Metrics and Their Generating Classes of
  Functions}.
\newblock \emph{Advances in Applied Probability}, 29\penalty0 (2):\penalty0
  429--443, 1997.

\bibitem[Nguyen et~al.(2009)Nguyen, Wainwright, and
  Jordan]{nguyen2009surrogate}
X.~Nguyen, M.~J. Wainwright, and M.~I. Jordan.
\newblock On surrogate loss functions and f-divergences.
\newblock \emph{The Annals of Statistics}, 37\penalty0 (2):\penalty0 876--904,
  2009.

\bibitem[Pan and Yang(2010)]{pan_survey_2010}
S.~J. Pan and Q.~Yang.
\newblock A {Survey} on {Transfer} {Learning}.
\newblock \emph{IEEE Transactions on Knowledge and Data Engineering},
  22\penalty0 (10):\penalty0 1345--1359, 2010.

\bibitem[Qui{\~n}onero-Candela et~al.(2008)Qui{\~n}onero-Candela, Sugiyama,
  Schwaighofer, and Lawrence]{quinonero2008dataset}
J.~Qui{\~n}onero-Candela, M.~Sugiyama, A.~Schwaighofer, and N.~D. Lawrence.
\newblock \emph{Dataset shift in machine learning}.
\newblock Mit Press, 2008.

\bibitem[Rakotomamonjy et~al.(2021)Rakotomamonjy, Flamary, Gasso, Alaya, Berar,
  and Courty]{rakotomamonjy2021optimal}
A.~Rakotomamonjy, R.~Flamary, G.~Gasso, M.~E. Alaya, M.~Berar, and N.~Courty.
\newblock Optimal transport for conditional domain matching and label shift.
\newblock \emph{Machine Learning}, pages 1--20, 2021.

\bibitem[Redko et~al.(2017)Redko, Habrard, and Sebban]{redko_theoretical_2017}
I.~Redko, A.~Habrard, and M.~Sebban.
\newblock Theoretical analysis of domain adaptation with optimal transport.
\newblock In \emph{Joint European Conference on Machine Learning and Knowledge
  Discovery in Databases}, pages 737--753. Springer, 2017.

\bibitem[Redko et~al.(2019)Redko, Courty, Flamary, and Tuia]{redko2019optimal}
I.~Redko, N.~Courty, R.~Flamary, and D.~Tuia.
\newblock Optimal transport for multi-source domain adaptation under target
  shift.
\newblock In \emph{The 22nd International Conference on Artificial Intelligence
  and Statistics}, pages 849--858. PMLR, 2019.

\bibitem[Redko et~al.(2020)Redko, Morvant, Habrard, Sebban, and
  Bennani]{redko_survey_2020}
I.~Redko, E.~Morvant, A.~Habrard, M.~Sebban, and Y.~Bennani.
\newblock A survey on domain adaptation theory.
\newblock \emph{CoRR}, abs/2004.11829, 2020.

\bibitem[Reid and Williamson(2011)]{reid2011information}
M.~D. Reid and R.~C. Williamson.
\newblock Information, divergence and risk for binary experiments.
\newblock \emph{Journal of Machine Learning Research}, 12:\penalty0 731--817,
  2011.

\bibitem[R{\'e}nyi(1961)]{renyi1961measures}
A.~R{\'e}nyi.
\newblock On measures of entropy and information.
\newblock In \emph{Proceedings of the Fourth Berkeley Symposium on Mathematical
  Statistics and Probability, Volume 1: Contributions to the Theory of
  Statistics}, pages 547--561. University of California Press, 1961.

\bibitem[Royden and Fitzpatrick(1988)]{royden1988real}
H.~L. Royden and P.~Fitzpatrick.
\newblock \emph{Real analysis}, volume~32.
\newblock Macmillan New York, 1988.

\bibitem[Saito et~al.(2017)Saito, Watanabe, Ushiku, and
  Harada]{saito_maximum_2017}
K.~Saito, K.~Watanabe, Y.~Ushiku, and T.~Harada.
\newblock Maximum {Classifier} {Discrepancy} for {Unsupervised} {Domain}
  {Adaptation}.
\newblock \emph{arXiv:1712.02560 [cs]}, 2017.

\bibitem[Saito et~al.(2019)Saito, Kim, Sclaroff, Darrell, and
  Saenko]{saito_semi-supervised_2019}
K.~Saito, D.~Kim, S.~Sclaroff, T.~Darrell, and K.~Saenko.
\newblock Semi-supervised {Domain} {Adaptation} via {Minimax} {Entropy}.
\newblock \emph{arXiv:1904.06487 [cs]}, 2019.

\bibitem[Santambrogio(2015)]{santambrogio2015optimal}
F.~Santambrogio.
\newblock Optimal transport for applied mathematicians.
\newblock \emph{Birk{\"a}user, NY}, 55\penalty0 (58-63):\penalty0 94, 2015.

\bibitem[Shen et~al.(2018)Shen, Qu, Zhang, and Yu]{shen_wasserstein_2018}
J.~Shen, Y.~Qu, W.~Zhang, and Y.~Yu.
\newblock Wasserstein distance guided representation learning for domain
  adaptation.
\newblock In \emph{Thirty-Second AAAI Conference on Artificial Intelligence},
  2018.

\bibitem[Shimodaira(2000)]{shimodaira_covariate_2000}
H.~Shimodaira.
\newblock Improving predictive inference under covariate shift by weighting the
  log-likelihood function.
\newblock \emph{Journal of Statistical Planning and Inference}, 90\penalty0
  (2):\penalty0 227--244, 2000.

\bibitem[Shu et~al.(2018)Shu, Bui, Narui, and Ermon]{shu_dirt-t_2018}
R.~Shu, H.~H. Bui, H.~Narui, and S.~Ermon.
\newblock A {DIRT}-{T} {Approach} to {Unsupervised} {Domain} {Adaptation}.
\newblock \emph{arXiv:1802.08735 [cs, stat]}, 2018.

\bibitem[Sion(1958)]{sion1958general}
M.~Sion.
\newblock On general minimax theorems.
\newblock \emph{Pacific J. Math.}, 8\penalty0 (1):\penalty0 171--176, 1958.

\bibitem[Steerneman(1983)]{steerneman1983total}
T.~Steerneman.
\newblock On the total variation and hellinger distance between signed
  measures; an application to product measures.
\newblock \emph{Proceedings of the American Mathematical Society}, 88\penalty0
  (4):\penalty0 684--688, 1983.

\bibitem[Sugiyama et~al.(2007)Sugiyama, Krauledat, and
  Müller]{sugiyama_covariate_2007}
M.~Sugiyama, M.~Krauledat, and K.-R. Müller.
\newblock Covariate {Shift} {Adaptation} by {Importance} {Weighted} {Cross}
  {Validation}.
\newblock \emph{Journal of Machine Learning Research}, 8\penalty0
  (May):\penalty0 985--1005, 2007.

\bibitem[Sun and Saenko(2016)]{sun_deep_2016}
B.~Sun and K.~Saenko.
\newblock Deep {CORAL}: {Correlation} {Alignment} for {Deep} {Domain}
  {Adaptation}.
\newblock \emph{arXiv:1607.01719 [cs]}, 2016.

\bibitem[Sun et~al.(2016)Sun, Feng, and Saenko]{sun_return_2016}
B.~Sun, J.~Feng, and K.~Saenko.
\newblock Return of frustratingly easy domain adaptation.
\newblock In \emph{Proceedings of the Thirtieth AAAI Conference on Artificial
  Intelligence}, AAAI’16, page 2058–2065, 2016.

\bibitem[Tachet~des Combes et~al.(2020)Tachet~des Combes, Zhao, Wang, and
  Gordon]{tachet2020domain}
R.~Tachet~des Combes, H.~Zhao, Y.-X. Wang, and G.~J. Gordon.
\newblock Domain adaptation with conditional distribution matching and
  generalized label shift.
\newblock In \emph{Advances in Neural Information Processing Systems},
  volume~33, pages 19276--19289, 2020.

\bibitem[Tan et~al.(2020)Tan, Peng, and Saenko]{tan2020class}
S.~Tan, X.~Peng, and K.~Saenko.
\newblock Class-imbalanced domain adaptation: an empirical odyssey.
\newblock In \emph{European Conference on Computer Vision}, pages 585--602.
  Springer, 2020.

\bibitem[Tong et~al.(2022)Tong, Garipov, Zhang, Chang, and
  Jaakkola]{tong2022adversarial}
S.~Tong, T.~Garipov, Y.~Zhang, S.~Chang, and T.~S. Jaakkola.
\newblock Adversarial support alignment.
\newblock In \emph{International Conference on Learning Representations}, 2022.
\newblock URL \url{https://openreview.net/forum?id=26gKg6x-ie}.

\bibitem[Wang and Deng(2018)]{wang_deep_2018}
M.~Wang and W.~Deng.
\newblock Deep visual domain adaptation: A survey.
\newblock \emph{Neurocomputing}, 312:\penalty0 135--153, 2018.

\bibitem[Weiss et~al.(2016)Weiss, Khoshgoftaar, and Wang]{weiss_survey_2016}
K.~Weiss, T.~M. Khoshgoftaar, and D.~Wang.
\newblock A survey of transfer learning.
\newblock \emph{Journal of Big Data}, 3\penalty0 (1):\penalty0 9, 2016.

\bibitem[Wilson and Cook(2020)]{wilson2020survey}
G.~Wilson and D.~J. Cook.
\newblock A survey of unsupervised deep domain adaptation.
\newblock \emph{ACM Transactions on Intelligent Systems and Technology (TIST)},
  11\penalty0 (5):\penalty0 1--46, 2020.

\bibitem[Wu et~al.(2019)Wu, Winston, Kaushik, and Lipton]{wu2019domain}
Y.~Wu, E.~Winston, D.~Kaushik, and Z.~Lipton.
\newblock Domain adaptation with asymmetrically-relaxed distribution alignment.
\newblock In \emph{International Conference on Machine Learning}, pages
  6872--6881. PMLR, 2019.

\bibitem[Zhang et~al.(2013)Zhang, Schölkopf, Muandet, and
  Wang]{zhang_domain_2013}
K.~Zhang, B.~Schölkopf, K.~Muandet, and Z.~Wang.
\newblock Domain {Adaptation} under {Target} and {Conditional} {Shift}.
\newblock In \emph{International {Conference} on {Machine} {Learning}}, pages
  819--827, 2013.

\bibitem[Zhang(2019)]{zhang_transfer_2019}
L.~Zhang.
\newblock Transfer {Adaptation} {Learning}: {A} {Decade} {Survey}.
\newblock \emph{arXiv:1903.04687 [cs]}, 2019.

\bibitem[Zhang et~al.(2019)Zhang, Liu, Long, and Jordan]{zhang_bridging_2019}
Y.~Zhang, T.~Liu, M.~Long, and M.~Jordan.
\newblock Bridging {Theory} and {Algorithm} for {Domain} {Adaptation}.
\newblock In \emph{International {Conference} on {Machine} {Learning}}, pages
  7404--7413, 2019.

\bibitem[Zhang et~al.(2020)Zhang, Long, Wang, and Jordan]{zhang2020localized}
Y.~Zhang, M.~Long, J.~Wang, and M.~I. Jordan.
\newblock On localized discrepancy for domain adaptation.
\newblock \emph{arXiv preprint arXiv:2008.06242}, 2020.

\bibitem[Zhao et~al.(2019)Zhao, Combes, Zhang, and Gordon]{zhao_learning_2019}
H.~Zhao, R.~T.~D. Combes, K.~Zhang, and G.~Gordon.
\newblock On {Learning} {Invariant} {Representations} for {Domain}
  {Adaptation}.
\newblock In \emph{International {Conference} on {Machine} {Learning}}, pages
  7523--7532, 2019.

\bibitem[Zhong et~al.(2021)Zhong, Fang, Liu, Lu, Yuan, and
  Zhang]{zhong2021does}
L.~Zhong, Z.~Fang, F.~Liu, J.~Lu, B.~Yuan, and G.~Zhang.
\newblock How does the combined risk affect the performance of unsupervised
  domain adaptation approaches?
\newblock In \emph{Proceedings of the 35th AAAI Conference on Artificial
  Intelligence}, 2021.

\bibitem[Zhuang et~al.(2020)Zhuang, Qi, Duan, Xi, Zhu, Zhu, Xiong, and
  He]{zhuang2020comprehensive}
F.~Zhuang, Z.~Qi, K.~Duan, D.~Xi, Y.~Zhu, H.~Zhu, H.~Xiong, and Q.~He.
\newblock A comprehensive survey on transfer learning.
\newblock \emph{Proceedings of the IEEE}, 109\penalty0 (1):\penalty0 43--76,
  2020.

\bibitem[Zolotarev(1984)]{zolotarev1984}
V.~M. Zolotarev.
\newblock Probability metrics.
\newblock \emph{Theory of Probability \& Its Applications}, 28\penalty0
  (2):\penalty0 278--302, 1984.

\end{thebibliography}
\newpage
\appendix
\section{Proofs for the different claims}
In this appendix, we provide proofs for the different claims made in the manuscript. We denote the indicator function of a set $A$ using the Iverson Bracket, \ie, $[\x \in A] = 1$ if $\x \in A$ and $0$ otherwise.
\explCrossEntropyLoss*
\begin{proof} 
Given $g$ as a scoring function and a point $\x \in \Xcal$, we have
\begin{align*}
    l(g(\x), h_g(\x)) &= \inf_{\y \in \Ycal}-\sum_{i=1}^K y_i(\log g(\x))_i\\
    &= - \max_i\log(g(\x)_i) = H_\infty(g(\x)) < H_\alpha(g(\x)) \quad\forall \alpha\geq 1.
    \end{align*}
\end{proof}
\explHingeLoss*
\begin{proof}
For a scoring function $g$ and a point $\x \in \Xcal$ we have
$$(1-\sgn{g(x)}g(\x))_+ = (1-\abs{g(x)})_+ \leq (1-yg(\x) \forall y \in \{-1,1\}$$,
where we used the the fact that the hinge loss is nonincreasing. The computation can be generalized to any loss function that is non-increasing and applied to $yg(\x)$.
\end{proof}
\propSourceGuidedUncertainty*
\begin{proof}
~\paragraph{Point 1} follows from bounding the infimum over $h \in \Hbb$ by the value for $h = h_g$, and that if $g$ is a classifier then $g = h_g$.
~\paragraph{Point 2} 
\begin{align*}
    \inf_{h \in \Hbb}\urisk[1,2]{\Hbb'}{h} 
    &= \inf_{h \in \Hbb} \inf_{h' \in \Hbb'}(\risk[1]{\Tcal}{h,h'}) + \risk[2]{\Scal}{h'}\\
    &= \inf_{h' \in \Hbb'} \inf_{h \in \Hbb}(\risk[1]{\Tcal}{h,h'}) + \risk[2]{\Scal}{h'}\\ 
    &= \inf_{h' \in \Hbb'}\risk[2]{\Scal}{h'}\quad{\text{using }}\Hbb \subseteq \tilde\Hbb
\end{align*}
where the last equality is obtained by setting $h$ to $h'$, which is possible since $\Hbb' \subseteq \Hbb$.
~\paragraph{Point 3} We have
\begin{align*}
        \risk{\Tcal}{h^*} + \risk{\Scal}{h^*}&\geq \risk{\Tcal}{h^*} + \urisk{\Hbb}{h^*}\quad\text{ (from point 1) }\\
        &= \inf_{h \in \Hbb} \risk{\Tcal}{h^*} + \risk{\Tcal}{h^*,h} + \risk{\Scal}{h}\\
        &\geq \inf_{h \in \Hbb} \risk{\Tcal}{h'} + \risk{\Scal}{h'} \quad\text{(triangle inequality)}\\
        &= \risk{\Tcal}{h^*} + \risk{\Scal}{h^*}.
    \end{align*}
    Hence $\risk{\Scal}{h^*} = \urisk{\Hbb}{h^*}$.
    For the second equality, we have
    \begin{align*}
        \risk{\Scal}{h_\Scal} &= \inf_{h \in \Hbb}\risk{\Scal}{h}\quad\text{(by definition of } h_\Scal)\\
        &= \inf_{h \in \Hbb}\urisk{\Hbb}{h}\quad\text{(by point 2)}\\
        &\leq \urisk{\Hbb}{h_\Scal}\\ 
        &\leq \risk{\Scal}{h_\Scal}\quad\text{(by point 1)}.
    \end{align*}
\end{proof}
\propSourceGuidedUncertaintyBound*
\begin{proof}
    Let $g \in \Ycal^\Xcal$. We have
    \begin{align*}
        \risk[1]{\Tcal}{g} - \sup_{h \in \Hbb}\left(\risk[2]{\Tcal}{h} - \risk[2]{\Scal}{h}\right)
        &=\risk[1]{\Tcal}{g} + \inf_{h \in \Hbb}\left(-\risk[2]{\Tcal}{h} + \risk[2]{\Scal}{h}\right)\\
        &= \inf_{h \in \Hbb} \risk[1]{\Tcal}{g} - \risk[2]{\Tcal}{h} + \risk[2]{\Scal}{h}\\
        &\leq \inf_{h \in \Hbb} \risk[1]{\Tcal}{g,h} + \risk[2]{\Scal}{h}\quad\text{(by assumption on loss functions)}
    \end{align*}
    On the other hand, by assumption, we have $\sup_{h \in \Hbb} \risk[2]{\Tcal}{h} - \risk[2]{\Scal}{h} \leq A(\Tcal,\Scal)$. Summing the last two inequalities gives the result.
\end{proof}

\paragraph{Proof for choices of loss functions in Proposition \ref{prop:sourceGuidedUncertaintyBound}}
\paragraph{Case of the cross entropy loss}
Let $l^1$ be chosen as the cross-entropy loss. For $\y_1,\y_2$ we have two basis vectors of $\Rset^K$:
$$l^1(g(\x),\y_1) - l^1(g(\x),\y_2) = (\y_2 - \y_1)^T\log(g(\x)) \leq \norm{\y_2 - \y_1}_1\norm{\log(g(\x))}_\infty$$
Assume $g_i(\x) = \frac{e^{a_i(\x)}}{\sum_{k=1}^Ke^{a_k(\x)}}$, where $g_i$ denotes the $i-$th component of $g$. Then we have
\begin{align}
\norm{\log(g(\x))}_\infty &= \max_{1\leq k\leq K} \abs{\log(g(\x))}\nonumber\\
&= \max_{1\leq k\leq K} -\log(g(\x))\quad(\text{since }0\leq g(\x)\leq 1)\nonumber\\
&= \max_{1\leq k\leq K} \log\left(\sum_{i = 1}^K e^{a_i(\x)}\right) - a_k(\x)\nonumber\\
&\leq \max_{1\leq k, i\leq K} a_i(\x) - a_k(\x) + \log K\\
&\leq 2R + \log K
\end{align}
Hence, the condition on losses holds for $l_1$ being the cross entropy loss and $l_2(\cdot,\cdot) = (2R + \log K)\norm{\cdot - \cdot}_1$.

\IMDProperties*
\begin{proof}
	\par For nonnegativity, for any two finite measures $\Qcal_1$ and $\Qcal_2$, and since $\Fbb$ contains the null function, we have
	$$\IMD_{\Fbb_0}(\Qcal_1,\Qcal_2) = \sup_{f \in \Fbb}\int f \dd\Qcal_1 - \int f \dd\Qcal_2 \geq \int 0 \dd\Qcal_1 - \int 0 \dd\Qcal_2 = 0.$$
	\par For the triangle inequality, for any three measures $\Qcal_1, \Qcal_2, \Qcal_3$, we have:
	\begin{align}
		\IMD_{\Fbb}(\Qcal_1,\Qcal_3) &= \sup_{f \in \Fbb}\int f\dd\Qcal_1 - \int f\dd\Qcal_3\nonumber\\
		&= \sup_{f \in \Fbb}\int f\dd\Qcal_1 - \int f\dd\Qcal_2 + \int f\dd\Qcal_2 - \int f\dd\Qcal_3\nonumber\\
		&\leq \sup_{f \in \Fbb}\int f\dd\Qcal_1 - \int f\dd\Qcal_2 + \sup_{f \in \Fbb}\int f\dd\Qcal_2 - \int f\dd\Qcal_3 \quad\text{(by sub-additivity of the supremum)},\nonumber
	\end{align}
	hence the IMD satisfies the triangle inequality.\\
	\par For the characterization of a null IMD, let $\Qcal_1,\Qcal_2$  be two measures. First, assume that $\Qcal_2 \geq \Qcal_1$. Then for any $f \in \Fbb$, since $\Fbb$ is nonnegative, we have $\int f \dd\Qcal_1 \leq \int f\dd\Qcal_2$. Hence $\IMD_\Fbb(\Qcal_1,\Qcal_2) \leq 0$. Due to the nonnegativity of the IMD, we conclude that $\IMD(\Qcal_1,\Qcal_2)=0$.
	Conversely, assume that $\IMD_\Fbb(\Qcal_1,\Qcal_2) = 0$, implying that $\forall f \in \Fbb,\quad \int f\dd(\Qcal_1 - \Qcal_2) \leq 0$. We will examine two richness assumptions on
	 $\Fbb$: containing the continuous functions over $\Xcal$ and containing indicators functions of measurable sets.
	 \begin{description}
	 	\item[Case when $\Fbb$ contains indicator functions] In this case, the choice $f = [\cdot \in A]$ where $A$ is a measurable set implies that $\Qcal_2(A) \geq \Qcal_1(A)$ for any set $A$ and allows to conclude.
	 	\item[Case when $\Fbb$ contains the continuous functions] 
	 	We will prove that for any compact set $K \subseteq \Xcal$, we have $\Qcal_1(K) \leq \Qcal_2(K)$, following the same proof of \citet[Chapter 21, Section 4, Proposition 11]{royden1988real}. Let $K \subseteq \Xcal$ be a compact set and let $\epsilon>0$. The measures, $\Qcal_1$ and $\Qcal_2$ are Radon measures in our case, hence by outer regularity of $\Qcal_2$, there is a neighborhood of $O$ of $K$ such that 
	 	\begin{equation}\label{eq:OminusKEpsilon}
	 	\Qcal_2(O\setminus K) \leq \epsilon.
	 	\end{equation}
	 	By the compact extension property \citep[Chapter 21, Section 2, Theorem 7]{royden1988real}, there exists a continuous function $0\leq f\leq 1$ over $\Xcal$ with compact support, such that
	 	\begin{align}
	 		f(\x) &= 1 \quad \forall \x \in K \label{eq:fUnitK}\\
	 		f(\x) &= 0 \quad \forall \x \in \Xcal\setminus O. \label{eq:fZeroOminusK}
	 	\end{align}
 		We have
	 	\begin{align}
	 		\Qcal_1(K) - \Qcal_2(K) &= \int_K f\dd\Qcal_1 - \int_K f\dd\Qcal_2\quad(\text{by property }\eqref{eq:fUnitK})\nonumber\\
	 		&= \int f\dd\Qcal_1 - \int f\dd\Qcal_2 + \int_{\Xcal\setminus K} f \dd\Qcal_2 - \int_{\Xcal\setminus K} f \dd\Qcal_1 \nonumber\\
	 		&\leq \int_{\Xcal\setminus K} f \dd\Qcal_2 - \int_{\Xcal\setminus K} f \dd\Qcal_1\quad(\text{since } f \in \Fbb\text{ and } \int f\dd(\Qcal_1-\Qcal_2) \leq 0) \nonumber\\
	 		&= \int_{O\setminus K} f \dd\Qcal_2 - \int_{O\setminus K} f \dd\Qcal_1\quad(\text{by property }\eqref{eq:fZeroOminusK})\nonumber\\
	 		&\leq \epsilon\quad(\text{by property \eqref{eq:OminusKEpsilon} and since $0\leq f\leq 1$})
	 	\end{align}
 		Hence, we have $\Qcal_1(K) - \Qcal_2(K) \leq \epsilon$ for an arbitrary $\epsilon>0$, implying that $\Qcal_1(K)\leq \Qcal_2(K)$ and allowing to conclude.
	 \end{description}
	 \par Concerning the densities of $\Qcal_2$ and $\Qcal_1$, since $\Qcal_2 - \Qcal_1$ is absolutely continuous, it has a Radon-Nikodym derivative that we denote $q$. We have $\Qcal_2 = \Qcal_2 - \Qcal_1 + \Qcal_1$. By the linearity of the Radon-Nikodym derivative, we have $q_2 = q + q_1$, implying that $q_2 \geq q_1$.
\end{proof}

\propBoundLocalizedIMD*
\begin{proof}
    We first use a generalization of \citet[Theorem]{zhang2020localized}. Let $h_1 \in \Hbb^{r_1}, h_2 \in \Hbb^{r_2}$. Then
    \begin{align*}
        \risk{\Tcal}{h_1} &\leq \risk{\Tcal}{h_1,h_2} + \risk{\Tcal}{h_2}\\
        &= \risk{\Tcal}{h_1,h_2} - \risk{\Scal}{h_2} + \risk{\Scal}{h_2} + \risk{\Tcal}{h_2}\\
        &\leq \risk{\Scal}{h_1}+ \risk{\Tcal}{h_1,h_2} - \risk{\Scal}{h_1,h_2} + \risk{\Scal}{h_2} + \risk{\Tcal}{h_2}\\
        &\leq \risk{\Scal}{h_1}+ \sup_{\substack{h_1 \in \Hbb^{r_1}\\h_2 \in \Hbb^{r_2}}}\left(\risk{\Tcal}{h_1,h_2} - \risk{\Scal}{h_1,h_2}\right) + \inf_{h_2 \in \Hbb^{r_2}}\left(\risk{\Scal}{h_2} + \risk{\Tcal}{h_2}\right)\\
        &= \risk{\Scal}{h_1}+ \IMD_{l(\Hbb^{r_1},\Hbb^{r_2})} + \inf_{h_2 \in \Hbb^{r_2}}\left(\risk{\Scal}{h_2} + \risk{\Tcal}{h_2}\right),
    \end{align*}
    where $$l(\Hbb^{r_1},\Hbb^{r_2}) \coloneqq \{\x\mapsto(h_1(\x), l(h_2(\x))); h_1 \in \Hbb^{r_1}, h_2 \in \Hbb^{r_2}\}$$ 
    \par Now we prove that $l(\Hbb^{r_1}, \Hbb^{r_2}) \subseteq \Fbb_{r_1 + r_2}$ 
    Given $h_i \in \Hbb_{r_i}$ for $i \in \{1,2\}$, we have
    \begin{align*}
        \risk{\Scal}{h_1,h_2} &\leq \risk{\Scal}{h_1} + \risk{\Scal}{h_2}\quad\text{(triangle inequality)}\\
        &\leq r_1 + r_2\quad(\text{definition of } \Hbb^{r_1}, \Hbb^{r_2}).
    \end{align*}
    meaning that $l(h_1,h_2) \in \Fbb_{r_1 + r_2}$. Since the last inequality holds for any  $l(h_1,h_2) \in  l(\Hbb^{r_1},\Hbb^{r_2})$, we conclude that $l(\Hbb^{r_1}, \Hbb^{r_2}) \subseteq \Fbb_{r_1 + r_2}$.
    Since the supremumm taken over $\Fbb_{r_1 + r_2}$ is at least equal to the supremum over its subset $l(\Hbb^{r_1}, \Hbb^{r_2})$, we have
    \begin{equation}
        \risk{\Tcal}{h_1} \leq \risk{\Scal}{h_1} + \IMD_{\Fbb_{r_1 + r_2}} + \inf_{h \in \Hbb^{r_2}}\risk{\Tcal}{h} + \risk{\Scal}{h} \quad \forall h_1 \in \Hbb^{r_1}.
    \end{equation}
    At this stage, we have bounds on the target domain that involve a source risk. Hence, applying our result from  Proposition \ref{prop:sourceGuidedUncertaintyBound}, allows to  replace the source risk by the source-guided uncertainty and to conclude.
\end{proof}

\propIMDDuality*
\begin{proof}
We have
$$\IMD_{F_\epsilon}(\Tcalx,\Scalx) = \sup_{f \in \Fbb}\inf_{\alpha \geq 0} \esp{\Tcalx}{f} - \esp{\Scalx}{f} + \alpha(\epsilon - \esp{\Scalx}{f}).$$
Applying the inf-sup inequality gives the first part of the result.
\par The result of equality for $\epsilon>0$  is an application of \citet[Chapter 8, Section 6, Theorem 1]{luenberger1997optimization}. We consider the optimization problem:
\begin{equation}
    \begin{aligned}
        \inf_{f \in \Fbb}&\quad \esp{\Scalx}{f} - \esp{\Tcalx}{f}\\
        \text{s.t.}&\quad \esp{\Scalx}{f} \leq \epsilon
    \end{aligned}
\end{equation}
which is equivalent to our problem up to a minus sign. We have
\begin{itemize}
    \item $\Fbb$ is a convex subset the vector space of real valued functions taking an argument in $\Xcal$.
    \item The value of the problem at the solution is finite since its unconstrained optimal value is finite by assumption (equal to $\IMD_\Fbb(\Tcalx,\Scalx)$).
    \item The null function $f_0 = 0 \in \Fbb$, and we have $\esp{\Scalx}{f_0} = 0 < \epsilon$.
    \item The objective function and constraints are all linear.
\end{itemize}
Hence, applying \citet[Chapter 8, section 6, Theorem 1]{luenberger1997optimization}, the value of the constrained problem at the optimum is equal to  $\max_{\alpha\geq 0}\inf_{f \in \Fbb}\esp{\Scalx}{f} - \esp{\Tcalx}{f} + \alpha(\esp{\Scalx}{f} - \epsilon)$, \ie, the maximum is achieved for some $\alpha^*\geq 0$.
\end{proof}

\corFLocalizedIsDivergence*
\begin{proof}
	Let $\epsilon>0$, and assume $\IMD_{\Fbb_\epsilon}$ is null. By Proposition \ref{prop:IMDDuality}, there exists $\alpha^*\geq 0$ such that for any $f \in \Fbb$, we have $\esp{\Tcalx}{f} \leq (1+\alpha^*)\esp{\Scalx}{f} - \epsilon \alpha^*.$ That holds in particular for $f$ chosen as the null function, which would result a contradiction unless $\alpha^*=0$. In that case, $\esp{\Tcalx}{f} \leq \esp{\Scalx}{f} \quad\forall f \in \Fbb$. This means that $\IMD(\Tcalx,\Scalx) =0$ and by Proposition \ref{prop:IMDProperties}, it implies that $\Scalx \geq \Tcalx$. Since $\Scalx$ and $\Tcalx$ are probability distributions, this implies that $\Tcalx = \Scalx$.
\end{proof}
\propBoundLocalizedIMDBetaPerClass*
\begin{proof}
	The proof follows the same reasoning we had through Propositions \ref{prop:BoundLocalizedIMD}, \ref{prop:IMDDuality} and Corollary \ref{cor:BoundLocalizedIMDBeta}, through the following steps:
	\begin{enumerate}
		\item Adapting the proof of Proposition \ref{prop:BoundLocalizedIMD} to $\Fbb_{\rbf_1}$ and $\Fbb_{\rbf_2}$ instead of $\Fbb_{r_1}$ and $\Fbb_{r_2}$.
		\item Adapting the proof of Proposition \ref{prop:IMDDuality}:
		$$\IMD_{F_\epsbf}(\Tcalx,\Scalx) = \sup_{f \in \Fbb}\inf_{\alphabf \in \Rset_+} \esp{\Tcalx}{f} - \esp{\Scalx}{f} + \sum_{k=1}^K\alpha_k(\epsilon_k - \esp{\Scal_{\Xcal|k}}{f})$$.
		Then the inf-sup upper bound follows similarly. Again, we can apply the same strong duality argument if $\epsilon_k>0$ for any $1\leq k\leq K$.
		\item The last step is to use the inf-sup upper bound to bound the IMD for an arbitrary $\betabf \in \Rset_+^K$.
	\end{enumerate}
\end{proof}

\corBoundLocalizedIMDBetaPerClassSplitting*
\begin{proof}
	Applying Proposition \ref{prop:BoundLocalizedIMDBetaPerClass} with the specified choice of $\rbf_1$ and $\rbf_2$, and with $\betabf$ verifying $\betabf^T\ones \leq \beta$, we get
	\begin{equation}
		\IMD_\Fbb\left(\Tcalx, \Scalx + \sum_{k=1}^K \beta_k\Scal_{\Xcal|k}\right) +\betabf^T(\rbf_1 + \rbf_2) \leq \IMD_\Fbb\left(\Tcalx, \Scalx + \sum_{k=1}^K \beta_k\Scal_{\Xcal|k}\right) + \beta (r_1 + r_2).
	\end{equation}
	Since this inequality holds for any choice of $\betabf^T\ones \leq \beta$, we can take the minimum which results in the DA bound (the minimum is achieved as the set $\{\betabf \in \Rset_+^K; \betabf^T\ones\leq \beta\}$ is compact).\\ As for the equality constraint, it is sufficient to notice that increasing any of the $\beta_k$ cannot increase the value  of $\IMD_\Fbb\left(\Tcalx, \Scalx + \sum_{k=1}^K \beta_k\Scal_{\Xcal|k}\right)$.
\end{proof}
\paragraph{Proof of the reweighting interpretation}

Let $\beta\geq 0$ and $\p$ be the vector of source class proportions. For $\betabf \geq 0$  verifying $\betabf^T\ones = \beta$, we have $$\sum_{k=1}^K \beta_k\Scal_{\Xcal|k} = \sum_{k=1}^K (p_k + \beta_k)\Scal_{\Xcal|k} = (1+\beta) \sum_k \tilde p_k \Scal_{\Xcal|k},$$
where $\tilde\p = \frac{\p + \beta}{1+\beta}$. Conversely, for any $\tilde \p \in \Delta_K$ such that $(1+\beta)\tilde \p\geq \p$, letting $\betabf = (1+\beta)\tilde\p - \p$, we have $\betabf^T\ones = \beta$. That results in the following set equality
$$\left\{\frac{\betabf + \p}{1+\beta}; \betabf\geq 0; \betabf^T\ones= \beta\right\} = \{\tilde \p \in \Delta_K; (1+\beta)\tilde\p \geq \p\},$$ which allows to conclude.

\propTargetShift*
\begin{proof}
Let $\Dcalxy = \Tcalxy = \Scalx$. Then
\begin{align}
    \IMD_\Fbb(\Tcalx, (1+\beta)\Scalx) &= \sup_{f \in \Fbb}\esp{\Tcalx}{f} - \esp{\Scalx}{f}\\
    &= \sup_{f \in \Fbb} \sum_{k=1}^K(q_k - (1+\beta)p_k)\esp{\Dcal_{\Xcal|k}}{f}.
\end{align}
Choosing $\beta \geq \max_k\left(\frac{q_k}{p_k}-1\right)_+$ will guarantee that the quantity $(q_k - (1+\beta)p_k)$ is nonpositive for any $1\leq k\leq K$. In this case, $f=0$ achieves the supremum and the IMD is null.
\par Likewise, we have 
\begin{align}
    \IMD_\Fbb\left(\Tcalx, \sum_{k=1}^K (p_k+\beta_k)\Scal_{\Xcal|k}\right) &= \sup_{f \in \Fbb}\esp{\Tcalx}{f} - \sum_{k=1}^K(p_k + \beta_k)\esp{\Scal_{\Xcal|k}}{f}\\
    &= \sup_{f \in \Fbb} (q_k - p_k - \beta_k)\esp{\Dcal_{\Xcal|k}}{f}
\end{align}
Choosing $\beta_k \geq (q_k - p_k)_+$ will guarantee that the quantity $(q_k - p_k - \beta_k)$ is nonpositive for any $1\leq k\leq K$. The rest follows as in the first case.
\par For the converse, assuming that $\Fbb$ is rich enough, then by Proposition \ref{prop:IMDProperties}, having $\IPM(\Tcalx, (1+\beta)\Scalx) = 0$ implies that $\sum_k (q_k - p_k - \beta p_k)\Dcal_{\Xcal|k} \leq 0$. 
By the assumption made on the family of sets $\{B_l\}_{l=1}^K$, we have for any $1\leq l\leq K$,
$$\sum_k (q_k - p_k - \beta p_k)\Dcal_{\Xcal|k}(B_l) = (q_l - p_l - \beta p_l)\Dcal_{\Xcal|l}(B_l) \leq 0.$$
Dividing by $\Dcal_{\Xcal|l}(B_l)$ implies that $\beta p_l \geq q_l - p_l$. Due to $\beta$'s nonnegativity, this means that $\beta \geq (\frac{q_l}{p_l} - 1)_+$.
\par For the per-class case, the same reasoning holds via considering $(q_k - p_k - \beta_k)\Dcal_{\Xcal|k}$ instead of $(q_k - p_k - \beta p_k)\Dcal_{\Xcal|k}$.
\end{proof}

\propDABoundTotalVariation*
\begin{proof}
	The result is an application of Proposition \ref{prop:BoundLocalizedIMD} for the two extreme cases $r_1 = r_2 = 1$ (in which case, the localization does not result in any restriction since the loss function $l$ is assumed to be bounded by 1 by assumption) and $r_1 = r_2 = 0$. We will then bound $\urisk{\Hbb}{g}$ and $\IPM_{\Fbb_{r_1 + r_2}}(\Tcalx,\Scalx)$ for the two extreme cases. We will denote by $\mu$ a measure dominating both $\Scalx$ and $\Tcalx$ (\eg, the Lebesgues measure).
	
	\paragraph{Bounding $\urisk{\Hbb}{h}$ and the ideal joint risk:} First, we have $\urisk{\Hbb}{h} = \inf_{h \in \Hbb} \int t(\x)l(g(\x),h(\x))) + l(h(\x), f_S(\x)) \dd\mu(\x)$. Given that $\Hbb$ has infinite capacity, we have
	$$\urisk{\Hbb}{h} = \int \inf_{h \in \Hbb}t(\x)l(g(\x),h(\x))) + l(h(\x), f_S(\x)) \dd\mu(\x)$$
	 so that we can reason on the integrand $\inf_{h \in \Hbb}t(\x) l(g(\x),h(\x)) + s(\x) l(h(\x), f_S(\x))$. For the sake of readability, we omit the dependence on $\x$ the point of interest.\\ 
	 At point $\x$, we first bound the integrand using the triangle inequality as 
	\begin{equation}\label{eq:infCapacityIntegrandBound}
		\inf_{h \in \Hbb}t\ l(g,h) + s\ l(h, f_S) \leq t\ l(g,h_g) + \inf_{h \in \Hbb}t\ l(h_g, h) + s\ l(h, f_S).
	\end{equation}
	If $h_g = f_S$, then setting $h = f_S = h_g$ achieves a minimum equal to 0. Else, depending on whether $t\geq s$, we can set $h = f_S$ or $h = h_g$,
	 leading to the equality
	 \begin{equation}\label{eq:minSTForm}
	 	\inf_{h \in \Hbb}t\ l(h_g, h) + s\ l(h, f_S) = \min(s,t)l(h_g,f_S).
	 \end{equation}
 	Combining \eqref{eq:infCapacityIntegrandBound} and \eqref{eq:minSTForm}, we have
 	\begin{equation}\label{eq:boundUncertaintyInfCapacity}
 		\urisk{\Hbb}{h} \leq \risk{\Tcal}{g,h_g} + \int \min(s,t)l(h_g,f_S)\dd\mu.
 	\end{equation}
 	For the other extreme case corresponding to $\Hbb^0$, by definition we have $\risk{\Scal}{h} = 0$ for $h \in \Hbb^0$, so that
 	\begin{equation}\label{eq:boundUncertaintyInfCapacityZero}
 		\urisk{\Hbb^0}{h} = \inf_{h \in \Hbb^0}\risk{\Tcal}{g,h} \leq \risk{\Tcal}{g,f_S}.
 	\end{equation}
 	Since $\urisk{\Hbb}{f_T}$ is the ideal joint risk, by a similar argument, we have
 	\begin{equation}\label{eq:BoundIdealJointRiskInfCapacity}
 		\inf_{h \in \Hbb}\risk{\Tcal}{h} + \risk{\Scal}{h} = \int \min(s,t)l(f_T,f_S)\dd\mu
 	\end{equation}
 	and
 	\begin{equation}\label{eq:BoundIdealJointRiskInfCapacityZero}
 		\inf_{h \in \Hbb^0}\risk{\Tcal}{h} + \risk{\Scal}{h} \leq \risk{\Tcal}{f_S}.
 	\end{equation}
 	\paragraph{Bounding $\IPM_\Fbb(\Tcalx,\Scalx)$:} We have
 	\begin{align}
 		\IMD_{\Fbb}(\Tcalx,\Scalx) &= \sup_{f \in \Fbb}\int f(t-s)\dd\mu\nonumber\\ 
 		&= \int (t-s)_+\dd\mu \quad(\text{choosing the indicator function }[t(\cdot)>s(\cdot)] \text{to achieve the supremum})\nonumber\\
 		&= \frac{1}{2} \int \abs{t-s} + (t-s)\dd\mu\quad(\text{since } u_+ = \frac{1}{2}(u + \abs{u})\text{ for }u \in \Rset)\nonumber\\ 
 		&= \frac{1}{2}\int\abs{t-s}\dd\mu = \frac{1}{2}d_1(\Tcalx,\Scalx)\quad(\text{since }\int s\dd\mu = \int t \dd\mu = 1) \label{eq:IMDInfCapacity}.
 	\end{align}
 	For the case of $\Fbb_0$, we have
 	\begin{align}
 		\IMD_{\Fbb_0}(\Tcalx,\Scalx) &= \sup_{f \in \Fbb_0}\int f\dd\Tcalx \quad(\text{by definition of }\Fbb_0 \text{ and the } \IMD)\nonumber\\
 		&= \sup_{f \in \Fbb_0} \int_{\supp\Scalx^c} f\dd\Tcalx + \int_{\supp\Scalx} f\dd\Tcalx\nonumber\\
 		&= \sup_{f \in \Fbb_0}\int_{\supp\Scalx^c} f\dd\Tcalx + \int_{\supp\Scalx} f\frac{t}{s}\dd\Scalx\quad(\text{since }s(\x)>0 \ \forall \x \in \supp\Scalx)\nonumber\\
 		&= \sup_{f \in \Fbb_0}\int_{\supp\Scalx^c} f\dd\Tcalx \quad(\text{since for any }f \in \Fbb_0,\ f\text{ is null }\Scalx-\text{almost everywhere})\nonumber\\
 		&= \Tcalx(\supp\Scalx^c) (\text{choosing the indicator function }[. \in \supp\Scalx^c] \text{to achieve the supremum}).
 		\label{eq:IMDInfCapacityZero}
 	\end{align}
	 \par Finally, combining \eqref{eq:boundUncertaintyInfCapacity}, \eqref{eq:BoundIdealJointRiskInfCapacity} and \eqref{eq:IMDInfCapacity} gives the first result. Likewise, combining \eqref{eq:boundUncertaintyInfCapacityZero}, \eqref{eq:BoundIdealJointRiskInfCapacityZero} and \eqref{eq:IMDInfCapacityZero} gives the second result.
\end{proof}

\corHdHIMD*
\begin{proof}
	In the first equality, for $\beta\geq 0$ we have
	\begin{align}
		1 - \IMD_{\HdH}(\Tcalx,(1+\beta)\Scalx) &= 1 - \sup_{f \in \HdH}\proba{\Tcalx}{f = 1} - (1+\beta)\proba{\Scalx}{f = 1}\nonumber\\
		&= \inf_{f \in \HdH}\proba{\Tcalx}{f = 0} + (1+\beta)\proba{\Scalx}{f = 1}.
	\end{align}
	Similarly, in the second one, for any $\betabf \in \Rset_+^K$ we have
	\begin{align}
		1 - \IMD_{\HdH}\left(\Tcalx,\Scalx + \sum_k\beta_k\Scal_{\Xcal|k}\right) &= 1 - \sup_{f \in \HdH}\proba{\Tcalx}{f = 1} - \proba{\Scalx}{f = 1} + \sum_k \beta_k\proba{\Scal_{\Xcal|k}}{f = 1}\nonumber\\
		&= \inf_{f \in \HdH}\proba{\Tcalx}{f = 0} + \sum_{k=1}^K\proba{\Scal_{\Xcal_k}}{f = 1}.
	\end{align}
	The result follows by taking the maximum over $\betabf$ such that $\sum_{k}\beta_k \leq \beta$.
\end{proof}

\propHSupport*
\begin{proof}
We have
\begin{align}
	\IMD_{\HdH_0}(\Tcalx,\Scalx) &= \sup_{f \in \HdH_0}\proba{\Tcalx}{f=1}\quad(\text{since } f \in \HdH_0 \Rightarrow \proba{\Scalx}{f=1} = 0)\nonumber\\
	&= \sup_{\substack{h_1,h_2 \in \Hbb\\\proba{\Scalx}{h_1\neq h_2} = 0}}\proba{\Tcalx}{h_1\neq h_2}\quad(\text{by definition of }\HdH)\nonumber\\
	&= 1 - \inf_{\substack{h_1,h_2 \in \Hbb\\\proba{\Scalx}{h_1= h_2} = 1}}\proba{\Tcalx}{h_1=h_2}\nonumber\\
	&\leq 1 - \proba{\Tcalx}{\bigcap_{{\substack{h_1,h_2 \in \Hbb\\\proba{\Scalx}{h_1= h_2} = 1}}}\{h_1 = h_2\}} \text{(the intersection of closed sets is closed hence measurable)}\nonumber\\
\end{align}
\end{proof}

\propOptimalTransportLocalized*
\begin{proof}
	Given that $l$ verifies the triangle inequality, for any $\{y_i\}_{i=1}^4 \subseteq \Ycal$ we have 
	\begin{equation}\label{eq:triangleIneqFourArgs}
		l(y_1,y_2) - l(y_3,y_4) \leq l(y_1,y_3) + l(y_3,y_4) + l(y_4,y_2) - l(y_3,y_4) = l(y_1,y_3) + l(y_4,y_2).
	\end{equation}
	Hence, given $\u,\v \in \Xcal$ and $h_1, h_2 \in \Hbb$, by property \eqref{eq:triangleIneqFourArgs}, we have
	\begin{align}
		&l(h_1(\u),h_2(\u)) - l(h_1(\v),h_2(\v))\leq l(h_1(\u),h_1(\v)) + l(h_2(\u),h_2(\v))\\
		&l(h_1(\v),h_2(\v)) - l(h_1(\u),h_2(\u))\leq l(h_1(\v),h_1(\u)) + l(h_2(\v),h_2(\u))\\
		\Rightarrow& \abs{l(h_1(\u),h_2(\u)) - l(h_1(\v),h_2(\v))}\leq l(h_1(\u),h_1(\v)) + l(h_2(\u),h_2(\v))\quad(\text{by symmetry of } l)\\
		&\leq \frac{1}{2}d(\u,\v) + \frac{1}{2}d(\u,\v). 
	\end{align}
	Thus the space $l(\Hbb,\Hbb) \subseteq \Fbb$, where $\Fbb$ is the space of nonnegative $1-$Lipschitz functions.

	Given a fixed $\betabf \in \Rset_+^K$, computing $\IMD_\Fbb\left(\Tcalx, \sum_{k=1}^K(p_k + \beta_k)\Scal_{\Xcal|k}\right)$ is equivalent to solving the following optimization problem.
	\begin{align}
		\sup_{f\in \Fbb_{\text{Lip}}}&\quad \int f(\u)\dd\Tcalx(\u) - (p_k+\beta_k)\int f(\v_k)\dd\Scal_{\Xcal|k}(\v)\\
		\text{s. t.}&\quad f(\u) - f(\v_k) \leq d(\u,\v) \forall \u, \v_k \in \Xcal; 1\leq l\leq K
	\end{align}
	The problem can then be re-written:
	\begin{align}
		\sup_{f\geq 0}&\quad \int f(\u)\dd\Tcalx(\u) - (p_k+\beta_k)\int f(\v_k)\dd\Scal_{\Xcal|k}(\v)\\
		\text{s. t.}&\quad \sup_{\u,\v_k \in\Xcal}f(\u) - f(\v_k) \leq d(\u,\v)\\
		&\quad \inf_{\u \in \Xcal} f(u)\geq 0, \inf_{\v_k \in \Xcal}f(\v_k)\geq 0 \quad\forall 1\leq \v_k \leq K
	\end{align}
	Now we apply \citet[Section 8.6, Theorem 1]{luenberger1997optimization}:
	\begin{itemize}
		\item The space of Lipschitz functions is a convex subset of the space of continuous valued functions of $\Xcal$.
		\item Any function that outputs a positive constant strictly verfies the inequality constraints.
		\item By compactness of $\Xcal$, the value of the supremum is finite.
	\end{itemize}
	Hence, strong Lagrangian duality holds: denoting $\mathfrak M+(A)$ the space of nonnegative finite measures on a set A, the value of the objective at the optimum is 
	\begin{align}
		&\inf_{\substack{\Pcal_k \in \mathfrak M_+(\Xcal^2)\\\Ucal, \Vcal_k \in \mathfrak M_+(\Xcal)}}\sup_{f \in \Fbb} \int f(\u)\dd\Tcalx(\u) - \int (p_k+\beta_k)f(\v_k)\dd\Scal_{\Xcal|k}(\v_k) \nonumber\\
		+& \int d(\u,\v_k) - f(\u) + f(\v_k) \dd\Pcal_k(\u,\v_k) + \int f(\v_k)\dd\Vcal_k + \int f(\u)\dd\Ucal\nonumber\\
		=& \inf_{\Pcal_k, \Vcal_k, \Ucal} \sum_{k}d(\u,\v_k)\dd\Pcal(\u,\v_k)\nonumber\\ 
		+& \sup_{f \in \Fbb} \int f(\u)\dd(\Tcalx(\u) - \sum_k \Pcal(\u,\v_k) + \Ucal(\u)) + \sum_k \int f(\v_k) \dd (\Pcal(\u,\v_k) - (p_k + \beta_k)\Scal_{\Xcal|k} + \Vcal_k).\nonumber
	\end{align}
	Following the same argument in \citet{santambrogio2015optimal}, the supremum in the latter formulation is finite if and only if we have
	\begin{align}
		\pi_1\#\sum_k\Pcal_k &= \Tcalx + \Ucal\\
		\pi_2\#\Pcal_k &= (p_k + \beta_k)\Scal_{\Xcal|k} - \Vcal_k.
	\end{align}
	The nonnegativity of $\Vcal_k$ and $\Ucal$ implies the formulation in the statement of the proposition.
	\par To transform the inequality constraint to an equality, notice that for any $1-$Lipschitz function $f$ that is non-negative on $\supp\Scalx$, setting $f$ to $f_+$ on $\supp \Tcalx \setminus \supp \Scalx$ cannot decrease the value of the objective, while keeping it in the feasible set of solutions (\ie, it stays nonnegative and $1-$Lipschitz). Hence, it suffices to impose the non-negativity constraint only on $\supp \Scalx$. This eliminates the variable $\Ucal$ and the inequality constraint on $\Tcalx$ becomes an equality.
\end{proof}
~\paragraph{Proof for support inclusion for optimal transport}
Let $\Fbb$ be the space of nonnegative $1-$Lipschitz functions over $\Xcal$.  For any $f \in \Fbb_0$, $f$ is nonnegative  and continuous with  $\int f\dd\Scalx = 0$, so $f$ is null on $\supp \Scalx$.  Hence, for any $\u \in \Xcal, \v \in \supp\Scalx$, we have
$f(\u) \leq f(\v)  + d(\u,\v) = d(\u,\v)$, implying that $f(\u) \leq d(\u, \supp\Scalx) \coloneqq \inf_{\v \in \supp\Scalx}d(\u,\v)$.
As a result, we have
$$\IMD_{\Fbb_0}(\u,\v) = \sup_{f \in \Fbb_0} \int_{\Tcalx}f(\u)\dd\Tcalx(\u) \leq \sup_{f \in \Fbb_0} \int d(\u,\Scalx)\dd\Tcalx(\u).$$
Moreover, the function $d(.,\Scalx)$ is 1-Lipschitz, nonnegative and is null on $\Scalx$, meaning that $d(.,\Scalx) \in \Fbb_0$. As a result, it achieves the supremum defining $\IMD_{\Fbb_0}$, hence the statement.

\end{document}